\documentclass{article}
\usepackage{graphicx} 
\usepackage{amsmath,amsfonts,amssymb,amsthm,color}
\usepackage[margin=1in]{geometry}
\usepackage{comment}
\usepackage{hyperref}

\newcommand{\R}{\mathbb{R}}
\newcommand{\E}{\mathbb{E}}
\newcommand{\p}{\mathbb{P}}

\newcommand{\bv}{\mathbf{v}}

\newcommand{\I}{\mathcal{I}}

\newtheorem{thm}{Theorem}

\newtheorem{lemma}{Lemma}
\newtheorem{assumption}{Assumption}

\newtheorem{rmk}{Remark}

\newtheorem{problem}{Problem}
\newtheorem{question}{Question}
\newtheorem{claim}{Claim}

\title{A theory of diversity for random matrices with applications to in-context learning of Schr\"odinger equations}
\author{Frank Cole, Yulong Lu and Shaurya Sehgal}
\date{}

\begin{document}

\maketitle

\abstract{We address the following question: given a collection $\{\mathbf{A}^{(1)}, \dots, \mathbf{A}^{(N)}\}$ of independent 
$d \times d$ random matrices drawn from a common distribution $\mathbb{P}$, what is the probability that the centralizer of $\{\mathbf{A}^{(1)}, \dots, \mathbf{A}^{(N)}\}$ is trivial? We provide lower bounds on this probability in terms of the sample size $N$ and the dimension $d$ for several families of random matrices which arise from the discretization of linear Schr\"odinger operators with random potentials. When combined with recent work on machine learning theory, our results provide guarantees on the generalization ability of transformer-based neural networks for in-context learning of  Schr\"odinger equations.}

\section{Introduction}\label{sec: intro}
Given a set of square matrices $\mathcal{A}:= \{\mathbf{A}^{(1)}, \cdots, \mathbf{A}^{(N)}\}$, we recall that the \textbf{centralizer} of $\mathcal{A}$ is defined as the set of all square matrices $\mathcal{M}$ such that every $\mathbf{X}\in \mathcal{M}$  commutes with each matrix in the collection, i.e., $ \mathbf{X}\mathbf{A}^{(i)} = \mathbf{A}^{(i)}\mathbf{X}, \; \forall i \in [N]
$. 
In this paper, we study the following question on random matrices, which we term the {\bf diversity} problem:

\begin{problem}[Diversity of random matrices]\label{prob: diversity}
    Given a probability distribution $\p$ on the set $\R^{d \times d}$ of real $d \times d$ matrices, and $N$ independent samples $\mathbf{A}^{(1)}, \dots, \mathbf{A}^{(N)}$ from $\p$, what is the probability that the centralizer of the set $\{\mathbf{A}^{(1)}, \dots, \mathbf{A}^{(N)}\}$ is trivial? In other words, we are asking for the probability whenever a matrix $\mathbf{X} \in \R^{d \times d}$ that satisfies 
$$ \mathbf{X}\mathbf{A}^{(i)} = \mathbf{A}^{(i)}\mathbf{X}, \; \forall i \in [N],
$$
we have that $M$ is a scalar multiple of the identity matrix. If there exists a sequence $\{c_N\}_{N \geq 1}$ with $\lim_{N \to \infty} c_N = 0$ such that
\begin{align}\label{eq: diversityrate}
    \p\left( \textrm{$\{\mathbf{A}^{(1)}, \dots, \mathbf{A}^{(N)}\}$ has trivial centralizer} \right) \geq 1-c_N,
\end{align}
we call the distribution \textbf{diverse}.
\end{problem}

The diversity problem is very simple to state, but it has been largely unstudied to the best of our knowledge. Indeed, while some works have studied related problems such as commutators of random matrices \cite{palheta2022commutators, jeong2022linear}, the primary focus of random matrix theory \cite{tao2012topics, edelman2005random} is to study limiting distributions of random matrices and their eigenvalues. In contrast, the diversity problem involves matrix centralizers, algebraic objects which arise naturally in the study of invariant subspaces and simultaneous diagonalizability of linear maps \cite{gantmakher2000theory}. The connection to invariant subspaces has led matrix centralizers to play a role in many areas of pure mathematics including Lie theory, representation theory, and algebraic geometry. Surprisingly, however, our study of the diversity problem is motivated by a connection to transformer neural networks and their generalization capabilities. This connection was noticed in the recent work \cite{cole2024context}; we describe this connection in the next subsection. Motivated by this connection, we investigate sufficient conditions on a matrix-valued probability distribution $\mathbb{P}$ to ensure that diversity holds in the sense of Problem \ref{prob: diversity}.

\subsection{Notational conventions}
Throughout the paper, non-boldface letters (e.g., $a$) to denote scalars or scalar-valued functions, we use lowercase boldface letters (e.g., $\mathbf{a}$) to denote vectors or entries of vectors, and we use uppercase bold letters such as $\mathbf{A}$ to denote matrices or entries of matrices. We also use calligraphic letters (e.g., $\mathcal{A}$) to denote sets or events. All vector spaces are over the real numbers unless explicitly stated otherwise. We use $\R$ to denote the set of real numbers and $\mathbb{N}$ to denote the set of natural numbers. Throughout the paper, $\p$ denotes a matrix-valued probability distribution.

\subsection{Transferability of transformers and task diversity}
The past few years have marked a paradigm shift in machine learning and data science, with less emphasis on domain-specific models and more emphasis on broad foundation models that can be applied to solve a variety of downstream tasks. These foundation models often adapt to new tasks through a mechanism known as \textit{in-context learning} \cite{garg2022can}, whereby a trained model first observes a short sequence of demonstrations of the learning task before making a new prediction. In-context learning offers a promising mathematical framework through which to rigorously analyze the behavior of foundation models. Some recent works have studied mathematical properties of in-context learning in simple settings such as linear regression \cite{ahn2023transformers, zhang2024trained, mahankali2023one, lu2025asymptotic, cole2024context}. Currently, however, important questions about in-context learning, such as the ability of models to generalize out-of-distribution, and the application of in-context learning models to scientific problems, remain largely undeveloped.

A recent work has begun to study this problem in a controlled theoretical setting. Specifically, \cite{cole2024context} studies the \textit{in-context learning of linear systems.} In this setting, the learning task is to predict from a vector $\mathbf{x} \in \R^{d}$ the corresponding solution $\mathbf{y} = \mathbf{A}\mathbf{x}$ to a linear system characterized by a coefficient matrix $\mathbf{A} \in \R^{d \times d}$. An in-context learning model takes two arguments as inputs: a length-$n$ sequence $\{(\mathbf{x}_i,\mathbf{y}_i)\}_{i=1}^{n}$ of demonstrations of the linear system (i.e., pairs $(\mathbf{x}_i,\mathbf{y}_i)$ with $\mathbf{y}_i = \mathbf{A}\mathbf{x}_i$), and an unlabeled vector $\mathbf{x}_{n+1} \in \R^{d \times d}$. The goal of the in-context learning model is to output the corresponding solution to the linear system $\mathbf{y}_{n+1} = \mathbf{A}\mathbf{x}_{n+1}$. Clearly, in order to predict the correct solution, the model must 1) infer the correct coefficient matrix from the sequence of examples, and 2) apply the coefficient matrix to the new vector. While solving linear systems may seem like a simple task, it turns out to capture important subtleties and nuances about the generalization of in-context learning models. In addition, it is rich enough to encompass  important problems in science: when a linear partial differential equation (PDE) is discretized in space, the resulting numerical problem is a linear system. Thus, solving linear systems in-context gives meaningful information about solving PDEs in-context, which is a very active research area \cite{yang2023context, yang2023prompting, calvello2024continuum, liu2023does, serrano2024zebra, mishra2025continuum} -- and a central motivation of this work.

As modern foundation models are built using the transformer architecture \cite{vaswani2017attention}, \cite{cole2024context} proposes to in-context learn linear systems using a simplified transformer architecture, characterized by a single self-attention layer. This model, parameterized by two weight matrices $\mathbf{P},\mathbf{Q} \in \R^{d \times d}$ and denoted $\textrm{TF}_{\mathbf{P},\mathbf{Q}}$, outputs a prediction for $\mathbf{y}_{n+1}$ given by
\begin{align}\label{eq: TF}
    \textrm{TF}_{\mathbf{P},\mathbf{Q}}\left(\{(\mathbf{x}_i,\mathbf{y}_i)\}_{i=1}^{n}, \mathbf{x}_{n+1} \right) = \mathbf{P} \left( \frac{1}{n} \sum_{i=1}^{n} \mathbf{y}_i \mathbf{x}_i^T \right) \mathbf{Q} \mathbf{x}_{n+1}.
\end{align}
The parameters $\mathbf{P}$ and $\mathbf{Q}$ of the model are trained by optimizing an empirical risk functional, averaging the model's squared prediction error over $N$ instances of the linear system problem. Specifically, we given $N$ coefficient matrices $\mathbf{A}^{(1)}, \dots, \mathbf{A}^{(N)}$, $N$ prompts $\{\{(\mathbf{x}_i^{(j)},\mathbf{y}_i^{(j)}\}_{j=1}^{n}\}_{i=1}^{N}$, and $N$ unlabeled vectors $\{\mathbf{x}_{n+1}^{(j)}\}_{j=1}^{N}$ we define the empirical risk functional by
\begin{align}\label{eq: emprisk}
\mathcal{R}_{\textrm{train}}\left(\textrm{TF}_{\mathbf{P},\mathbf{Q}} \right) = \frac{1}{N} \sum_{j=1}^{N} \left\| \textrm{TF}_{\mathbf{P},\mathbf{Q}} \left(\{(\mathbf{x}_i^{(j)},\mathbf{y}_i^{(j)}\}_{i=1}^{n}, \mathbf{x}_{n+1}^{(j)} \right) - \mathbf{y}_{n+1}^{(j)} \right\|^2.
\end{align}
In addition, we take the perspective of \cite{cole2024context} and assume that the coefficient matrices and vectors are random variables drawn from probability distributions. Define the learned parameters $\widehat{\mathbf{P}},\widehat{\mathbf{Q}}$ by
\begin{align*}
   \textrm{TF}_{\widehat{\mathbf{P}},\widehat{\mathbf{Q}}} \in \textrm{arg} \min_{\mathbf{P},\mathbf{Q}}   \mathcal{R}_{train}\left(\textrm{TF}_{\mathbf{P},\mathbf{Q}} \right).
\end{align*}
In abuse of notation, we write $\widehat{\textrm{TF}}$ in place of $\textrm{TF}_{\widehat{\mathbf{P}},\widehat{\mathbf{Q}}}$.
Given a new prompt $\{\mathbf{x}_i,\mathbf{y}_i\}_{i=1}^{n}$ corresponding to a new coefficient matrix $\mathbf{A}$, and a new vector $\mathbf{x}_{n+1}$, we would like to understand how well the trained model $\widehat{\textrm{TF}}$ predicts the solution $\mathbf{y}_{n+1}$, and how the error depends on the sample sizes $n$ and $N$. This error is defined by an $L^2$ loss measured uniformly across coefficient matrices:
\begin{align*}
    \textrm{Err}(\widehat{\textrm{TF}}) := \sup_{\|\mathbf{A}\| \leq 1} \E_{\mathbf{x}_1, \dots, \mathbf{x}_{n+1}} \left[ \left\|\widehat{\textrm{TF}} \left(\{\mathbf{x}_i,\mathbf{y}_i\}_{i=1}^{n}, \mathbf{x}_{n+1} \right)-\mathbf{y}_{n+1} \right\|^2 \right]
\end{align*}
It is not clear \textit{a priori} that the error above should tend to zero as the sample size increases. For instance, it may be that the training coefficient matrices $\mathbf{A}^{(1)}, \dots, \mathbf{A}^{(N)}$ have a shared structure (e.g., a common eigenbasis), and in this case, it is unclear if the generalization of the model is restricted to linear systems whose coefficient matrix shares this structure. This is especially relevant to the in-context learning of PDEs, where diverse training data may be computationally expensive to produce. Perhaps surprisingly, \cite{cole2024context} links the decay of the strong generalization error to the notion of diversity introduced in Problem \ref{prob: diversity}. The result is quoted below.

\begin{thm}\label{thm: strongerrorbd}
    Consider the set $\mathcal{S}(\p) := \{\mathbf{A}^{(1)}(\mathbf{A}^{(2)})^{-1}: \mathbf{A}^{(1)}, \mathbf{A}^{(2)} \in \textrm{supp}(\p)\}.$ If the set $\mathcal{S}(\p)$ has trivial centralizer, then, with high probability over the training set, the prediction error of $\widehat{\textrm{TF}}$ tends to zero as the sample size tends to $\infty$:
    \begin{align*}
        \textrm{Err}(\widehat{\textrm{TF}}) =  o_{n,N}(1), \; \; \; \textrm{with probability $\geq 1 - o_{n,N}(1).$}
    \end{align*}
\end{thm}
Theorem \ref{thm: strongerrorbd} demonstrates that if a transformer is trained to solve $N$ instances of the linear system problem with coefficient matrices $\mathbf{A}^{(1)}, \dots, \mathbf{A}^{(N)}$, then the model is able to generalize to solve arbitrary linear systems precisely when the centralizer of the set $\mathcal{S}(\p)$ is trivial. This connection motivates us to study sufficient conditions on the distribution $\p$ under which the centralizer of $\p$ is trivial. Note that if $\textrm{supp}(\p)$ contains a multiple of the identity matrix, then $\textrm{supp}(\p) \subset \mathcal{S}(\p)$. In this case, it suffices to prove the triviality of the centralizer of $\textrm{supp}(\p)$, which is usually easier to work with than $\mathcal{S}(\p)$. In general, it is reasonable to assume that the support of $\p$ contains a multiple of the identity matrix since, at the level of the empirical loss, this can be achieved by augmenting the training set with an example of the form $(\{(\mathbf{x}_i,\mathbf{x}_i)\}_{i=1}^{N}, \mathbf{x}_{n+1})$. Additionally, since we do not have access to the true distribution $\p$ in practice, we instead work with the empirical distribution $\widehat{\p} = \frac{1}{N} \sum_{i=1}^{N} \delta_{\mathbf{A}^{(i)}}$, where $\mathbf{A}^{(1)}, \dots, \mathbf{A}^{(N)}$ are iid samples from $\p$. In this case, validating the assumptions of Theorem \ref{thm: strongerrorbd} corresponds precisely to solving the diversity problem for the distribution $\p$. Having established the connection between the diversity problem and the transferability of transformers, the remainder of this paper sets out to verify the diversity of a broad family of matrix distributions.

What assumptions are reasonable to assume of the target distribution? We note that if the entries of the matrix are IID random variables from a continuous density, then one might suspect that, for sufficiently large $N$, the set of coefficient matrices with trivial centralizer has probability zero. However, matrix distributions arising from practical applications (such as solving PDEs) may be discrete in nature, due to numerical quantization or localization exhibited by physical phenomena. For such non-absolutely continuous distributions, it is not clear what additional assumptions suffice to ensure a distribution is diverse. Moreover, for the in-context learning of PDEs, the matrix distribution is determined by the distribution on the PDE coefficients, as well as the discretization scheme. We would therefore like to understand what PDE coefficient distributions, and what spatial discretization methods, induce a matrix distribution which satisfies the diversity property. We summarize these points in the following two questions.

\begin{question}
    Given a (possibly discrete) matrix-valued probability distribution $\p$, what conditions must hold to ensure that $\p$ is diverse in the sense of Problem \ref{prob: diversity}, and what is the rate with respect to the number $N$ of samples?
\end{question}

\begin{question}
    Given a Schr\"odinger operator of the form 
    \begin{equation}\label{eq: ellipticPDE}
       u(x) \mapsto -\Delta u(x) + V(x) u(x)
    \end{equation}
     with random potentials and zero boundary conditions, discretized in space to obtain a matrix-valued distribution $\p$, what assumptions on the distribution of the potential $V(x)$, and the numerical discretization scheme, ensure that $\p$ is diverse in the sense of Problem \ref{prob: diversity}?
\end{question}

\subsection{Summary of main contributions} 
We highlight the main contributions of the paper as follows. 
\begin{itemize}
    \item Motivated by the structure of Schr\"odinger operators \eqref{eq: ellipticPDE}, we consider matrix distributions whose samples take the form $\mathbf{A} = \mathbf{K} + \mathbf{V},$ where $\mathbf{K}$ is a deterministic matrix and $\mathbf{V}$ is a random matrix. Under suitable assumptions on the distribution of $\mathbf{V}$, our first result proves a lower bound on the probability that the \textit{augmented} sampled set $\{ \mathbf{A}^{(1)}, \dots, \mathbf{A}^{(N)}, \mathbf{K}\}$ has a trivial centralizer, which converges to 1 exponentially fast in $N$; see Theorem \ref{thm: main}.
    \item Moving beyond the augmented sample set, we consider matrix distributions whose samples take the form $\mathbf{A} = \mathbf{K} + \mathbf{V},$ where $\mathbf{K}$ is deterministic and $\mathbf{V}$ is a random \textit{diagonal} matrix. Under mild additional assumptions, we prove that  the probability that the centralizer of $\{\mathbf{A}^{(1)}, \dots, \mathbf{A}^{(N)}\}$ is trivial converges to $1$ exponentially fast in $N.$; see Theorem \ref{thm: 2}. 
    \item We apply Theorems \ref{thm: main} and \ref{thm: 2} to prove the diversity of matrix distributions arising from spatial discretizations of a class of random Schr\"odinger operators defined by \eqref{eq: ellipticPDE}. When combined with Theorem \ref{thm: strongerrorbd}, these results suggest that transformer-based machine learning models trained to solve elliptic equations can be transferred to solve arbitrary linear PDEs. 
    \item To validate our theoretical results, we conduct several numerical experiments where linear transformers are trained to solve random instances of a linear elliptic PDE and then tested on different instances of the same type of PDE. Thanks to the diversity of the training data, our simulations show that a single trained transformer can adapt both to shifts on the distribution of the PDE coefficients and to changes in the numerical discretization schemes.
\end{itemize}

\section{Diversity results for matrix distributions}

\subsection{Diversity of an augmented sample set}
In this section, we state assumptions under which a matrix-valued distribution $\p$ is diverse.  A sample from $\p$ will take the form $\mathbf{A} = \mathbf{K} + \mathbf{V}$, where $\mathbf{K}$ is a deterministic matrix and $\mathbf{V}$ is a random matrix. This motivates us to consider matrices which are given by random perturbations of a fixed matrix. 

\begin{assumption}\label{assum: main}
    A sample from the distribution $\p$ is of the form $\mathbf{A} = \mathbf{K} + \mathbf{V}$, where
    \begin{enumerate}
        \item $\mathbf{K} \in \R^{d \times d}$ is a deterministic symmetric matrix with distinct eigenvalues.
        \item $\mathbf{V} \in \R^{d \times d}$ is a random matrix such that the following holds: there exists an enumeration $\mathbf{u}_1, \dots, \mathbf{u}_M$ of linear independent eigenvectors of $\mathbf{K}$, and a constant $c \in (0,1),$ such that for each $k \in [d-1]$, we have
        \begin{align*}
            \p \left(\mathbf{u}_k^T \mathbf{V} \mathbf{u}_{k+1} = 0 \right) =: c < 1.
        \end{align*}
    \end{enumerate}
\end{assumption}
The assumption above is strongly motivated by the discretization of elliptic operators of the form $u \mapsto -\nabla \cdot (a(x)\nabla u) + V(x) u$: the matrices $\mathbf{K}$ and $\mathbf{V}$ correspond to discretizations of Laplacian operator and the multiplication operator induced by the random potential $V(x)$ respectively.

We sample $\mathbf{A}^{(1)}, \dots, \mathbf{A}^{(N)}$ from $\mathbb{P}$ and seek to bound from below the probability that the centralizer of $\mathcal{S}_N := \{\mathbf{A}^{(1)}, \dots, \mathbf{A}^{(N)}\}$ has a trivial centralizer. In this section, we study a slightly modified problem by studying the centralizer of the \textit{augmented} sample set $\mathcal{S}_N \cup \{\mathbf{K}\},$ where we augment the sample set by the deterministic matrix $\mathbf{K}$.  Under Assumption \ref{assum: main}, we prove the following bound concerning the augmented sample set.

\begin{thm}\label{thm: main}
    Let $\p$ satisfy Assumption \ref{assum: main} and $\mathbf{A}^{(1)}, \dots, \mathbf{A}^{(N)}$ be iid samples from $\p$. Define $\mathcal{S}_N = \{\mathbf{A}^{(1)}, \dots, \mathbf{A}^{(N)}\}.$ Then
    \begin{align*}
        \p \left( \textrm{$\mathcal{S}_N \cup \{\mathbf{K}\}$ has trivial centralizer} \right) \geq 1 - (d-1)c^N.
    \end{align*}
\end{thm}
\begin{proof}
    Let $\mathbf{X} \in \R^{d \times d}$ belong to the centralizer of $\mathcal{S}$. Then, in particular, $[\mathbf{X},\mathbf{K}] = 0,$ which implies that $\mathbf{X}$ and $\mathbf{K}$ are simultaneously diagonalizable. Since $\mathbf{K}$ has distinct eigenvalues, the orthogonal matrix which diagonalizes $\mathbf{K}$ is unique up to permutations of its columns. Hence, if we define $\mathbf{U} = \begin{bmatrix}
        \mathbf{u}_1 & \dots & \mathbf{u}_d
    \end{bmatrix},$ it must hold that $\mathbf{Y} := \mathbf{U}^T \mathbf{X} \mathbf{U}$ is a diagonal matrix. Now, let us write $\mathbf{A}^{(i)} = \mathbf{K} + \mathbf{V}^{(i)}$ and define $\mathbf{W}^{(i)} = \mathbf{U}^T \mathbf{V}^{(i)} \mathbf{U}$. Since the commutator is equivariant under change of basis, we have for each $i \in [N],$
    \begin{align*}
        0 = [\mathbf{X},\mathbf{A}^{(i)}] &= [\mathbf{X},\mathbf{K} + \mathbf{V}^{(i)}] \\
        &= [\mathbf{X},\mathbf{V}^{(i)}] \\
        &= [\mathbf{Y},\mathbf{W}^{(i)}].
    \end{align*}
    Since $\mathbf{Y}$ is diagonal, writing $\mathbf{Y} = \textrm{diag}(\mathbf{y}_1, \dots, \mathbf{y}_d)$, the entries of the commutator $[\mathbf{Y},\mathbf{W}^{(i)}]$ can be expressed as
    \begin{align*}
        0 =[\mathbf{Y},\mathbf{W}^{(i)}]_{j,k} = (\mathbf{y}_j-\mathbf{y}_k)\mathbf{W}_{j,k}^{(i)}.
    \end{align*}
    In particular, for each $k \in [d-1]$, we have
    \begin{align}\label{eq: commutator}
        0 = [\mathbf{Y},\mathbf{W}^{(i)}]_{k,k+1} = (\mathbf{y}_k-\mathbf{y}_{k+1})\mathbf{W}_{k,k+1}^{(i)}.
    \end{align}
    By assumption, $\mathbf{W}^{(i)}_{k,k+1} = \mathbf{u}_k^T \mathbf{V}^{(i)} \mathbf{u}_{k+1}$ is equal to zero with probability $c \in (0,1)$, hence
    \begin{align*}
        \p \left( \bigcup_{i=1}^{N} \{\mathbf{W}_{k,k+1}^{(i)} \neq 0\} \right) &= 1 -  \left( \bigcap_{i=1}^{N} \{\mathbf{W}_{k,k+1}^{(i)} = 0\} \right) \\
        &= 1 - c^N,
    \end{align*}
    where we used the independence of $\mathbf{W}^{(1)}, \dots, \mathbf{W}^{(N)}$ in the last line. Let $\mathcal{E}_N$ denote the event 
    \begin{align*}
        \mathcal{E}_N = \bigcap_{k=1}^{d-1} \bigcup_{i=1}^{N} \{\mathbf{W}_{k,k+1}^{(i)} \neq 0 \}.
    \end{align*}
    Then, combining the above argument with a union bound over $k \in [d-1]$, $\mathcal{E}_N$ holds with probability $1 - (d-1) c^N.$ In addition, on the event $\mathcal{E}_N,$ we must have $\mathbf{y}_k = \mathbf{y}_{k+1}$ for each $k \in [d-1]$; to see this, note that on the event $\mathcal{E}_N,$ for each $k \in [d-1],$ there must be some index $i \in [N]$ where $\mathbf{W}_{k,k+1}^{(i)} \neq 0$, hence the right-hand side of Equation \eqref{eq: commutator} can only equal zero if $\mathbf{y}_k = \mathbf{y}_{k+1}$. But this forces $\mathbf{Y} = \mathbf{U}^T \mathbf{X} \mathbf{U}$ to be a multiple of the identity matrix, hence $\mathbf{X}$ is a multiple of the identity matrix. This proves the desired claim.
\end{proof}

Theorem \ref{thm: main} answers a variant of Problem \ref{prob: diversity} for a broad class of matrix distributions. To facilitate the proof, we must assume that the sample set is augmented by constant part $\mathbf{K}$ of the distribution $\p.$ In the application of our theory to the in-context learning of PDEs, this is not a limiting assumption, as the matrix $\mathbf{K}$ represents a physical prior which is known to the learner. In Theorem \ref{thm: 2}, we will state results for the vanilla sample set, thus determining sufficient conditions under which Problem \ref{prob: diversity} has an affirmative answer. While the constant $c$ in Theorem \ref{thm: main} is defined rather implicitly, we provide concrete examples of matrix distributions for which this constant admits explicit estimates; see Section \ref{sec: applications} for further details.

\begin{rmk} Theorem \ref{thm: main} shows that if $\p$ satisfies Assumption \ref{assum: main}, then any set of $N$ iid samples from $\p$, augmented with the single matrix $\mathbf{K}$, has a trivial centralizer with exponentially-in-$N$ high probability. Instead of considering the sample set augmented by the matrix $\mathbf{K}$, we could assume that $A \sim \p$ satisfies $\mathbf{A} = \mathbf{K}$ with nontrivial probability, and that conditioned on $A \neq \mathbf{K}$, $\mathbf{u}_k^T(\mathbf{A}-\mathbf{K})\mathbf{u}_{k+1} = 0$ with nontrivial probability for each $k \in [d-1]$. We chose to present the results in the way that required the simplest possible assumptions.
\end{rmk}

\subsection{Diversity of the vanilla sample set}
We now return to the original setting of Problem \ref{prob: diversity}: we have $N$ iid samples from a matrix distribution $\p$ and we want to lower bound the probability that the centralizer of these matrices is trivial. As in the previous subsection, we adopt the assumption that a sample from $\p$ takes the form $\mathbf{A} = \mathbf{K} + \mathbf{V}$ with $\mathbf{K}$ being a deterministic matrix. However, without the assumption that the sample set is augmented by the matrix $\mathbf{K},$ we need stronger assumptions on the distribution $\p$; we state these assumptions below.

\begin{assumption}\label{assum: 2}
    A sample from the distribution $\p$ is of the form $\mathbf{A} = \mathbf{K} + \mathbf{V},$ where
    \begin{enumerate}
        \item $\mathbf{K}$ is a deterministic matrix such that there exists a permutation $\pi$ of the indices $[d]$ for which $\mathbf{K}_{\pi(k),\pi(k+1)} \neq 0$ for all $k \in [d-1].$, 
        \item $\mathbf{V} \in \R^{d \times d}$ is a random \textbf{diagonal} matrix with iid entries, and the distribution of the entries of $\mathbf{V}$ satisfies
        \begin{align*}
            \p\left(\{\bv_1-\bv_2 = \bv_3-\bv_4\} \cup \{\bv_1-\bv_2 = 0\} \cup \{\bv_3-\bv_4 = 0\} \right) \leq c_V
        \end{align*}
        for some constant $c_V \in (0,1).$
    \end{enumerate}
\end{assumption}
Under Assumption \ref{assum: 2}, we prove the following result on the diversity of $\p$.

\begin{thm}\label{thm: 2}
    Let $\p$ satisfy Assumption \ref{assum: 2} and let $\mathcal{S}_N = \{\mathbf{A}^{(1)}, \dots, \mathbf{A}^{(N)}\}$ comprise $N$ iid samples from $\p$. Then
    \begin{align*}
        \p\left(\textrm{$\mathcal{S}_N$ has trivial centralizer} \right) \geq 1 - d(d-1)c_V^{N/2}.
    \end{align*}
\end{thm}

\begin{proof}
    Write $\mathbf{A}^{(i)} = \mathbf{K} + \mathbf{V}^{(i)}$ with $\mathbf{V}^{(i)} = \textrm{diag}(\bv_1^{(i)}, \dots, \bv_d^{(i)}\}.$ We proceed in two steps.
    
    \paragraph{Step 1:} We first show that, with probability at least $1 - d(d-1)c_V^{N/2}$, any matrix $\mathbf{X}$ which commutes with $\mathbf{A}^{(i)}$ for all $i \in [N]$ is a diagonal matrix. For $1 \leq j \neq k \leq d$ and $1 \leq i < \ell \leq N$, define the events
    \begin{align}\label{eq: event1}
        E_{j,k}^{i,\ell} = \left\{\bv_j^{(i)} - \bv_k^{(i)} = \bv_j^{(\ell)} - \bv_k^{(\ell)} \right\} \cup \left\{\bv_j^{(i)} - \bv_k^{(i)} = 0 \right\} \cup \left\{ \bv_j^{(\ell)} = \bv_k^{(\ell)} = 0\right\}.
    \end{align}
    Note that $\mathbb{P}(E_{j,k}^{i,\ell}) = c_V$, where $c_V$ is the constant defined in Assumption \ref{assum: 2}. Also, define the event
    \begin{align}\label{eq: event2}
        \mathcal{E} := \bigcup_{1 \leq j \neq k \leq d} \bigcap_{1 \leq i < \ell \leq N} E_{j,k}^{i,\ell}.
    \end{align}
    We make the following claims:
    \begin{enumerate}
        \item On the complement of $\mathcal{E}$, the centralizer of $\{\mathbf{A}^{(1)}, \dots, \mathbf{A}^{(N)}\}$ is contained in the set of diagonal matrices.
        \item $\mathbb{P}(\mathcal{E}) \leq d(d-1)c_V^{N/2}.$
    \end{enumerate}
    To prove 1), note that any $\mathbf{X}$ that commutes with $\{\mathbf{A}^{(1)}, \dots, \mathbf{A}^{(N)}\}$ satisfies
    \begin{align*}
        0 = [\mathbf{A}^{(i)},\mathbf{X}] = [\mathbf{K} + \mathbf{V}^{(i)},\mathbf{X}] = [\mathbf{K},\mathbf{X}] + [\mathbf{V}^{(i)},\mathbf{X}].
    \end{align*}
    In particular, the function $i \mapsto [\mathbf{V}^{(i)},\mathbf{X}]$ is constant. Since $\mathbf{V}^{(i)}$ is diagonal, this commutator is given explicitly by
    \begin{equation}\label{eq: diagcommutator}
        [\mathbf{V}^{(i)},\mathbf{X}]_{j,k} = (\bv_j^{(i)} - \bv_k^{(i)})\mathbf{X}_{j,k}.
    \end{equation}
    On the complement of $\mathcal{E}$, it holds that whenever $j \neq k$, there exist two indices $i, \ell \in [N]$ such that $\bv_j^{(i)} - \bv_k^{(i)} \neq \bv_j^{(\ell)}-\bv_k^{(\ell)}$ and neither $\bv_j^{(i)} - \bv_k^{(i)}$ nor $\bv_j^{(\ell)}-\bv_k^{(\ell)}$ are equal to zero. Combined with Equation \eqref{eq: diagcommutator}, this implies that on the complement of $\mathcal{E}$, $\mathbf{X}_{j,k} = 0$ whenever $j \neq k.$ To prove 2), note that for each $j \neq k$, we have
    \begin{align*}
        \mathbb{P} \left(\bigcap_{1 \leq i < \ell \leq N} E_{j,k}^{i,\ell} \right) \leq  \mathbb{P} \left(\bigcap_{(1,2), \dots, (N-1,N)} E_{j,k}^{i,\ell} \right) \leq c_V^{N/2},
    \end{align*}
    where the second intersection is taken only over the $N/2$ mutually disjoint pairs of indices $\{(1,2),(3,4),\dots,(N-1,N)\}$, and the final inequality uses independence and the fact that $\mathbb{P}(E_{j,k}^{i,\ell}) = c_V$ (note that if $N$ is odd, we can simply discard the $N^{\textrm{th}}$ sample and run the previous argument with $N$ replaced by $N-1$). By the union bound, we conclude that
    $$ \mathbb{P}(\mathcal{E}) \leq \sum_{j\neq k} \mathbb{P} \left(\bigcap_{1 \leq i < \ell \leq N} E_{j,k}^{i,\ell} \right) \leq d(d-1)c_V^{N/2}.
    $$
    \paragraph{Step 2:} We complete the proof by showing that on the complement of $\mathcal{E}$, the centralizer of $\{\mathbf{A}^{(1)}, \dots, \mathbf{A}^{(N)}\}$ can only contain scalar matrices. We have already shown that on the complement of $\mathcal{E}$, the centralizer of $\{\mathbf{A}^{(1)}, \dots, \mathbf{A}^{(N)}\}$ can only contain diagonal matrices. If $\mathbf{X} \in \R^{d \times d}$ is diagonal and belongs to the centralizer of $\{\mathbf{A}^{(1)}, \dots, \mathbf{A}^{(N)}\}$, we have
    \begin{align*}
        0 = [\mathbf{A}^{(i)},\mathbf{X}] = [\mathbf{K} + \mathbf{V}^{(i)},\mathbf{X}] = [\mathbf{K},\mathbf{X}] + [\mathbf{V}^{(i)},\mathbf{X}] = [\mathbf{K},\mathbf{X}],
    \end{align*}
    where the last equality uses the fact that $[\mathbf{V}^{(i)},\mathbf{X}] = 0$ since $\mathbf{V}^{(i)}$ and $\mathbf{X}$ are both diagonal. This implies that for each pair of indices $(j,k)$ with $j \neq k$,
    \begin{align*}
        [\mathbf{K},\mathbf{X}]_{j,k} = \mathbf{K}_{j,k}(\mathbf{x}_j-\mathbf{x}_k) = 0,
    \end{align*}
    where $\{\mathbf{x}_1, \dots, \mathbf{x}_d\}$ denote the entries of $\mathbf{X}$. In particular, for each $j \in [d-1]$, we have
    $$ \mathbf{K}_{\pi(j),\pi(j+1)}(\mathbf{x}_{\pi(j)}-\mathbf{x}_{\pi(j+1)}) = 0,
    $$
    where $\pi$ denotes the permutation defined in Assumption \ref{assum: 2}. Since $\mathbf{K}_{\pi(j),\pi(j+1)} \neq 0$ for all $j \in [d-1]$, we have $\mathbf{x}_{\pi(j)} = \mathbf{x}_{\pi(j+1)}$ for all $j \in [d-1].$ This implies that $\mathbf{X}$ is a multiple of the identity matrix.
\end{proof}

\begin{rmk}
In contrast to Theorem \ref{thm: main}, Theorem \ref{thm: 2} establishes a bound on the centralizer of the sample set without augmenting by $\mathbf{K}$. However, the assumptions required for Theorem \ref{thm: 2} are stronger than those required for Theorem \ref{thm: main} since Theorem \ref{thm: 2} requires the random part of the matrix distribution $\p$ to be supported on the set of diagonal matrices. In addition, the bound in Theorem \ref{thm: 2} has a larger constant factor and slower rate with than the bound in Theorem \ref{thm: main}. We note also that the constant $c$ in Theorem \ref{thm: main} has the potential to be much smaller than the constant $c$ in Theorem \ref{thm: 2}. As an example, if the random matrix $\mathbf{V}$ is diagonal with iid entries from $\textrm{Ber}_p\{a,b\}$ with $0 < a < b$ and $p \in (0,1)$, then the constant in Theorem \ref{thm: main} can be taken as $c = \frac{1}{\sqrt{1+\frac{2}{3}d p(1-p)}}$, whereas the optimal constant in Theorem \ref{thm: 2} is equal to $c = 1-2p^2(1-p)^2.$
\end{rmk}

\section{Applications for diversity of discrete Schr\"odinger operators}\label{sec: applications}
We consider the implications of Theorems \ref{thm: main} and \ref{thm: 2} for matrix distributions arising from spatial discretizations of elliptic differential operators. Specifically, we consider the Schr\"odinger operator
\begin{equation}\label{eq: schrodinger}
  u(x) \mapsto (-\Delta + V(x))u(x)  
\end{equation}
on the domain $[0,1]^D$. We consider two common discretizations, outlined below.

\subsection{Schr\"odinger operators under finite difference discretization}\label{subsec: FD}
Under the finite difference method, the domain is discretized into a set of evaluation points, and functions are encoded by their pointwise evaluations. We will consider a uniform grid in $\R^D$ with $1/M$ spacing in each direction. Hence, the total number of evaluation points, and thus the dimension of our linear system, will be $d = M^D.$

Beginning in dimension $D = 1$, we define the fixed evaluation points $\{\mathbf{x}_1, \dots, \mathbf{x}_M\} \subset (0,1)$ where $\mathbf{x}_i = \frac{i}{M}.$ The finite difference representation of the Laplacian operator is the $M \times M$ matrix
\begin{equation}\label{eq: laplacianFD} -\mathbf{\Delta}_{\textrm{FD},1} = M^2 \begin{pmatrix}
    -2 & 1 & 0 & \dots & 0 \\
    \vdots & \vdots &\ddots &\vdots & \vdots \\
    0 & \dots & 0 & 1 & -2
\end{pmatrix}.
\end{equation}
The finite difference representation of the multiplication operator with respect to the potential $V(x)$ is simply a diagonal matrix with entries
\begin{equation}
    \mathbf{V} = \begin{pmatrix}
        V(x_1) & & \\
        & \ddots & \\
        & & V(x_M)
    \end{pmatrix}.
\end{equation}
When modeling the potential $V(x)$ as a random function, we assume that there exist positive numbers $0 < a < b$ such that, for each evaluation point $\mathbf{x}_k$, the potential $V(\mathbf{x}_k)$ takes values $a$ and $b$ with probabilities $p$ and $1-p$ respectively for some $p \in (0,1)$, and that the collection of potential values $V(\mathbf{x}_1), \dots, V(\mathbf{x}_p)$ are jointly independent. This discretization of the Schr\"odinger operators  induces a matrix-valued distribution $\p$, where $A \sim \p$ takes the form $\mathbf{A} = -\mathbf{\Delta}_{\textrm{FD},1} + \mathbf{V}$, where $\mathbf{V} = \textrm{diag}(\bv_1, \dots, \bv_M)$ with $\bv_1, \dots, \bv_M \sim \textrm{Ber}_p\{a,b\}$ for some $p \in (0,1).$

In dimension $D > 1$, we assume that the cube $[0,1]^D$ is discretized by a uniform grid $\{\mathbf{x}_1, \dots, \mathbf{x}_{M^D}\}$ with $1/d$-spacing. In this setting, the matrix representation of Laplacian operator can be represented in terms of tensor products of the matrix $\mathbf{\Delta}_{\textrm{FD},1}$ and the identity matrix:
\begin{equation}\label{eq: higherlaplacian}
    \mathbf{\Delta}_{\textrm{FD},D} = \sum_{i=1}^{D} \left(\bigotimes_{j=1}^{i-1} \mathbf{I}_{M} \right) \otimes \mathbf{\Delta}_{\textrm{FD},1} \otimes \left(\bigotimes_{j=i+1}^{M} \mathbf{I}_{M} \right).
\end{equation}
To illustrate, when $D=2$, the matrix takes the form
\begin{align*}
    \mathbf{\Delta}_{\textrm{FD},2} = \mathbf{I}_M \otimes \mathbf{\Delta}_{\textrm{FD},1} + \mathbf{\Delta}_{\textrm{FD},1} \otimes \mathbf{I}_M,
\end{align*}
and when $D = 3$, the matrix takes the form
\begin{align*}
    \mathbf{\Delta}_{\textrm{FD},3} = \mathbf{I}_M \otimes \mathbf{I}_M \otimes \mathbf{\Delta}_{\textrm{FD},1} + \mathbf{I}_M \otimes \mathbf{\Delta}_{\textrm{FD},1} \otimes \mathbf{I}_M + \mathbf{\Delta}_{\textrm{FD},1} \otimes \mathbf{I}_M \otimes \mathbf{I}_M.
\end{align*}
To make the multi-dimensional setting tractable, we assume that the potential function can be represented as a sum of separable potential functions:
\begin{align}\label{eq: separablepotential}
    V(x_1, \dots, x_D) = \sum_{k=1}^{m} \prod_{i=1}^{D} V_{k_i}(x_i).
\end{align}
The matrix representation of a separable potential is simply the Kronecker product of the matrix representations of each factor:
\begin{align*}
    V_{k_1, \dots, k_D} = \bigotimes_{i=1}^{D} V_{k_i}.
\end{align*}
It follows that the matrix representation of the potential in Equation \eqref{eq: separablepotential} is 
\begin{align}\label{eq: potentialmatrix}
    \mathbf{V} = \sum_{k=1}^{m} \bigotimes_{i=1}^{D} \mathbf{V}_{k_i}.
\end{align}
When modeling the potential as a random variable, we assume that the separable factors are IID and that each point evaluation is uniformly distributed on $\{a,b\}$ Thus, under finite difference discretization in the multi-dimensional setting, a sample $A \sim \p$ takes the form
\begin{align}\label{eq: pfd}
    \mathbf{A} = -\mathbf{\Delta}_{\textrm{FD},D} + \bigotimes_{i=1}^{D} \mathbf{V}_{k_i}
\end{align}
where $\mathbf{\Delta}_{\textrm{FD},D}$ is defined in Equation \eqref{eq: higherlaplacian} and $\mathbf{V}_{k_i} = \textrm{diag}(\bv_{k_i}^{(1)}, \dots, \bv_{k_i}^{(M)})$ with $\bv_{k_i}^{(j)} \sim \textrm{Ber}_p\{1,2\}.$ Having defined the matrix distribution $\p_{\textrm{FD}}$ induced by the finite difference discretization of random Schr\"odinger operators, we state  the diversity of $\p_{\textrm{FD}}.$ as an application of  Theorem \ref{thm: main}. 

\begin{thm}\label{thm: FD}
    Let $\p_{\textrm{FD}}$ be the distribution of the random matrix defined in Equation \eqref{eq: pfd}. Assume $M \geq \frac{9}{2p(1-p)}$ and $D \geq 1$ is arbitrary. Then $\p_{\textrm{FD}}$ satisfies Assumption \ref{assum: main} with $\mathbf{K} = -\mathbf{\Delta}_{\textrm{FD},D}$ and constant $c =\frac{2}{\sqrt{1+\frac{2}{3}M p (1-p)}}.$ Hence, if $\mathbf{A}^{(1)}, \dots, \mathbf{A}^{(N)}$ are iid samples from $\p_{\textrm{FD}}$, then the augmented sample set $\{\mathbf{A}^{(1)}, \dots, \mathbf{A}^{(N)}, -\mathbf{\Delta}_{\textrm{FD},D}\}$ has trivial centralizer with probability at least $1- \left(M^D-1 \right) \left( \frac{2}{\sqrt{1+\frac{2}{3}M p (1-p)}} \right)^N.$ In addition, when $D = 1$, the constant can be reduced by a factor of $2$ to $c = \frac{1}{\sqrt{1+\frac{2}{3}M p (1-p)}},$ and the result holds for any $M \geq 2.$
\end{thm}
We defer the proof to Section \ref{sec: proofs}. Theorem \ref{thm: FD} shows that under finite difference discretization, the centralizer of the augmented sample set of random matrices induced by Schr\"odinger operators is trivial with exponentially-in-$N$ high probability. In addition, we notice a surprising phenomenon with respect to the parameter $M$, where the rate at which the failure probability tends to zero \textit{increases} as $M$ increases. Since $M$ is related to the dimension of the linear system, this represents a "blessing of dimensionality", where a larger $M$ leads to a more disordered matrix distribution, thus increasing the probability that the centralizer of the sample set is trivial. It should be noted that the Bernoulli parameter $p$ has thus far been chosen independently of $M$. In practice, to ensure that the potential function is well-defined in the continuum limit, one may want to choose $p$ as a function of $M$ (note that $p$ determines the smoothness of the potential, with $p = 0$ corresponding to a constant function). Theorem \ref{thm: FD} shows that the blessing of dimension phenomenon persists as long as $\lim_{M \rightarrow \infty} Mp = \infty.$ In Section \ref{sec: numerics}, we present numerical experiments to investigate the role of the parameter $p$. Finally, we remark that the condition that $M \geq \frac{9}{2p(1-p)}$ is required for technical reasons in the proof when $D > 1,$ but we conjecture that this condition can be removed with a more careful analysis.

To complement Theorem \ref{thm: FD}, we prove the following result which lifts the assumption that the sample set is augmented by the Laplace matrix, at the expense of a weaker bound.

\begin{thm}\label{thm: FD2}
    Let $\p_{\textrm{FD}}$ be the distribution of the random matrix defined in Equation \eqref{eq: pfd} with $M \geq 2$ and $D \geq 1.$ Then, if $\mathcal{S}_N = \{\mathbf{A}^{(1)}, \dots, \mathbf{A}^{(N)}\}$ comprises $N$ iid samples from $\p_{\textrm{FD}},$ then we have
    \begin{align*}
        \p_{\textrm{FD}} \left(\textrm{$\mathcal{S}_N$ has trivial centralizer} \right) \geq 1 - M(M-1) \left(1-2p^2(1-p)^2 \right)^{N/2}.
    \end{align*}
\end{thm}
We defer the proof to Section \ref{sec: proofs}. Theorem \ref{thm: FD2} proves that the distribution $\p_{\textrm{FD}}$ is diverse in the sense of Problem \ref{prob: diversity} with an exponential rate. In addition, the estimate in Theorem \ref{thm: FD2} is independent of the spatial dimension $D$. Compared to Theorem \ref{thm: FD}, the base of the exponent in Theorem \ref{thm: FD2} is much closer to 1, and we do not see the blessing of dimensionality, where the bound improves as $M \rightarrow \infty,$ in Theorem \ref{thm: FD2}.

\subsection{Schr\"odinger operators under finite element discretization}\label{subsec: FEM}
Under the finite element method, functions are discretized according to their projections onto a finite-dimensional subspace spanned by piecewise linear functions. In dimension $D > 1,$ the matrix distribution induced by finite element discretization of random Schr\"odinger operators is more complicated than the distribution associated to the finite difference discretization. Therefore, in this work we restrict our analysis of the finite element method to the case $D = 1.$ In addition, we are only able to prove results on the centralizer of the augmented sample set (corresponding to the setting of Theorem \ref{thm: main}). We leave it as an open problem for future work to extend the results of this subsection to $D > 1,$ and to lift the augmentation assumption.

As in the previous subsection, we first outline the construction in dimension one and then extend to higher dimensions. As with the finite difference method, we begin with a uniform mesh on the interval $[0,1]$ with spacing $\frac{1}{M}$ with $d \geq 3$; such a mesh defines $d-1$ subintervals of length $\frac{1}{d-1}:$
\begin{align*}
    \mathcal{I}_k = \left[\frac{k-1}{d-1}, \frac{k}{d-1} \right], \; k \in \{0,1, \dots, d-1\}.
\end{align*}
For each $1 < k < d-1,$ we define the $k^{\textrm{th}}$ hat function:
\begin{align}
    \phi_k(x) = \begin{cases}
        (d-1) \left( x-\frac{k-1}{d-1} \right), \; x \in \I_k \\
        (d-1) \left( \frac{k+1}{d-1} - x \right), \; x \in \I_{k+1} \\
        0, \; \textrm{otherwise.}
    \end{cases}
\end{align}
We also define the hat functions $\phi_0$ and $\phi_{M-1}$ at the boundary by
\begin{align}
        \phi_k(x) = \begin{cases}
        (d-1) \left( x-\frac{k-1}{d-1} \right), \; x \in \I_k \\
        (d-1) \left( \frac{k+1}{d-1} - x \right), \; x \in \I_{k+1} \\
        0, \; \textrm{otherwise,}
    \end{cases}
\end{align}
and
\begin{align}
        \phi_{M-1}(x)
    \begin{cases}
        (d-1\left(x-\frac{d-2}{d-1} \right), \; x \in \I_d \\
        0, \; \textrm{otherwise.}
    \end{cases}
\end{align}
The span of $\{\phi_k\}_{k=0}^{d-1}$ forms a $d$-dimensional subspace of continuous functions, and we consider discretizing the Schr\"odinger operator \eqref{eq: schrodinger} by projecting onto this basis. In more detail, the matrix representation of the one-dimensional Laplacian under finite element discretization, denoted by $\mathbf{\Delta}_{\textrm{FEM},1}$ is given by
\begin{align}\label{eq: FEMlaplace}
    \langle \mathbf{e}_i, \mathbf{\Delta}_{\textrm{FEM},1} \mathbf{e}_j \rangle = \int_{0}^{1} \phi_{i-1}'(x) \phi_{j-1}'(x) dx,
\end{align}
where $i,j \in [d]$ and $e_i, e_j$ denote the $i^{\textrm{th}}$ and $j^{\textrm{th}}$ standard basis vectors respectively. Note that the matrices $\mathbf{\Delta}_{\textrm{FEM},1}$ and $\mathbf{\Delta}_{\textrm{FD},1}$ are equal up to a constant factor. The matrix representation of the potential function has entries given by
\begin{equation}
    \mathbf{V}_{ij} = \int_{0}^{1} \phi_{i-1}(x) \phi_{j-1}(x) V(x) dx.
\end{equation}
In this setting, we model the potential function as piecewise constant on each subinterval $\I_k$:
\begin{equation}\label{eq: pwconstantpotential}
    V(x) = \sum_{k=1}^{M-1} \bv_k \mathbf{1}_{\I_k}(x).
\end{equation}
As in the previous section, we assume the constant values $\bv_1, \dots, \bv_{k-1}$ are independently sampled from $\textrm{Ber}_p\{a,b\}.$ A simple calculation proves the following.
\begin{lemma}
    When the potential $V(x)$ is given by Equation \eqref{eq: pwconstantpotential}, the matrix representation of the multiplication operator corresponding to $V(x)$ under the finite element discretization is given by
    \begin{align}\label{eq: FEMpotentialmatrix}
        \frac{1}{6(d-1)} \begin{pmatrix}
        2\bv_1 & \bv_1 & & & \\
        \bv_1 & 4(\bv_1+\bv_2) & \bv_2 & & \\
        & \ddots & \ddots & \ddots & \\
        & &\bv_{d-2} & 4(\bv_{d-2}+\bv_{d-1}) & \bv_{d-1} \\
        & & & \bv_{d-1} & 2\bv_{d-1}
    \end{pmatrix}.
    \end{align}
\end{lemma}
Thus, in dimension one, the finite element discretization Schr\"odinger operator \eqref{eq: schrodinger} induces a matrix distribution whose samples are given by
\begin{align}\label{eq: pfem}
    \mathbf{A} = -\mathbf{\Delta}_{\textrm{FEM},1} + \mathbf{V},
\end{align}
where $\mathbf{\Delta}_{\textrm{FEM},1}$ is given by Equation \eqref{eq: FEMlaplace} and $\mathbf{V}$ is given by Equation \eqref{eq: FEMpotentialmatrix}. While the finite element discretization of the Schr\"odinger operator can be generalized to higher spatial dimensions, analyzing the resulting matrix distribution requires new techniques beyond the scope of this work. Our main result of this section focuses on the centralizer of the augmented sample set.
\begin{thm}\label{thm: FEM}
    Let $\p_{\textrm{FEM}}$ denote the distribution of the random matrix defined in Equation \eqref{eq: pfem} with $D=1$ and $M \geq 5$.Then $\p_{\textrm{FEM}}$ satisfies Assumption \ref{assum: main} with $\mathbf{K} = -\mathbf{\Delta}_{\textrm{FEM},1}$ and constant $c = \frac{1}{\sqrt{1+2p(1-p)}}.$ Hence, if $\mathbf{A}^{(1)}, \dots, \mathbf{A}^{(N)}$ are iid samples from $\p_{\textrm{FEM}},$ then the augmented sample set $\{\mathbf{A}^{(1)}, \dots, \mathbf{A}^{(N)}, \mathbf{\Delta}_{\textrm{FEM},1}\}$ has trivial centralizer with probability at least $1-(M-1) \left( \frac{1}{\sqrt{1+2p(1-p)}} \right)^N.$
\end{thm}
We defer the proof to Section \ref{sec: proofs}. Theorem \ref{thm: FEM} demonstrates that under finite element discretization in 1D, the centralizer of the augmented sample set of random matrices induced by Schr\"odinger operators is trivial with exponentially-in-$N$ high probability. While the stated bound does not improve as $M \rightarrow \infty$ (as was observed in Theorem \ref{thm: FD}), we conjecture that this is an artifact of our analysis, and that a similar bound to the one in Theorem \ref{thm: FD} can be proven for the finite element discretization. We discuss this further in the proof of Theorem \ref{thm: FEM}

\section{Numerical experiments}\label{sec: numerics}
In this section we numerically validate the task diversity condition for a class of task matrices arising from the discretization of Schr\"odinger operators. We also investigate the in-domain and out-of-domain generalization performance of pre-trained transformers for learning linear discrete Schr\"odinger operators in context. For in-context learning experiments, we use the linear transformer architecture whose prediction is described in Equation \eqref{eq: TF}, whose parameters consist of two learnable weight matrices $\mathbf{P},\mathbf{Q} \in \R^{d \times d}.$ We refer the reader to \cite{cole2024context} for a more detailed study of this architecture. The transformer parameters are trained by minimizing the empirical risk $\mathcal{R}_{\textrm{train}}$ using stochastic gradient descent. All codes used for the numerical experiments are available below. 

\href{https://github.com/LuGroupUMN/Diversity-of-random-matrices/}{https://github.com/LuGroupUMN/Diversity-of-random-matrices/}



\begin{figure}\label{fig: rankcondition1D}
    \centering
    \includegraphics[width=0.6\linewidth]{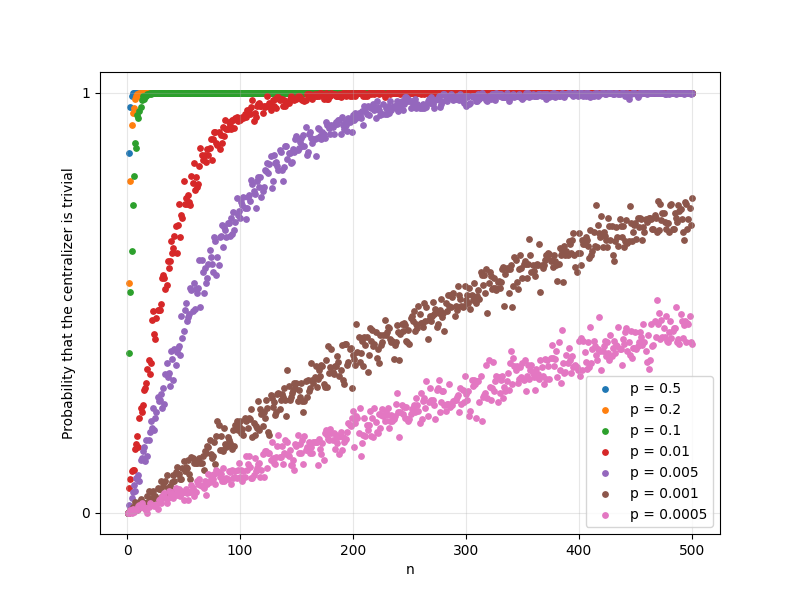}
    \caption{
    The probability that the centralizer of  discrete random 1D Schr\"odinger matrices—whose potential entries are sampled from ($\mathrm{Ber}_p\{1,2\}$)—is trivial, shown as a function of the sample size $N$. Different colors correspond to different values of the Bernoulli parameter $p$.}
    
    \label{fig: 1}
\end{figure}


\subsection{Experiments for 1D Schr\"odinger operators}\label{subsec: 1dexperiments}
Consider the system $\mathbf{A} \mathbf{x}=\mathbf{y}$, where the entries of $x \sim N(0,1)$ and $\mathbf{A}$ represents a discretization of the operator in (\ref{eq: schrodinger}) for $D=1$.

\subsubsection{Verifying the diversity condition numerically}\label{subsubsec: 1ddiversity}
We first present a linear algebraic method to directly verify whether the centralizer of a given set of matrices $\{\mathbf{A}^{(1)}, \dots, \mathbf{A}^{(N)}\}$ has trivial centralizer. For $\mathbf{A} \in \R^{d \times d},$ we let $\textrm{Vec}(\mathbf{A}) \in \R^{d^2}$ denote the vectorization of $\mathbf{A},$ i.e., the $d^2 \times 1$ column vector obtained by stacking the columns of $\mathbf{A}$ on top of each other. Observe that two matrices $\mathbf{A}, \mathbf{X} \in \R^{d \times d}$ commute if and only if the linear system
\begin{align*}
    (\mathbf{I}_d \otimes \mathbf{A} - \mathbf{A} \otimes \mathbf{I}_d) \textrm{Vec}(\mathbf{X}) = \mathbf{0}
\end{align*}
holds. By the rank-nullity theorem, it follows that a collection of matrices $\{\mathbf{A}^{(1)}, \dots, \mathbf{A}^{(N)}\}$ has trivial centralizer if and only if the matrix
\begin{align*}
    \begin{bmatrix}
        \mathbf{I}_d \otimes \mathbf{A}^{(1)} - \mathbf{A}^{(1)} \otimes \mathbf{I}_d \\
        \vdots \\
        \mathbf{I}_d \otimes \mathbf{A}^{(N)} - \mathbf{A}^{(N)} \otimes \mathbf{I}_d
    \end{bmatrix} \in \R^{Nd^2 \times d^2}
\end{align*}
has rank $d^2-1.$ This allows us to directly verify whether the centralizer of a given set of matrices is trivial by computing the rank of the associated $Nd^2 \times d^2$ coefficient matrix. 


To apply this method, for each $N$ we sample the matrices $\{A^{(1)}, \dots, A^{(N)}\}$ with $A^{(i)} = -\Delta_{\textrm{FD},1} + \textrm{diag}(v_1^{(i)}, \dots, v_d^{(i)}),$ where $d=5$, $\Delta_{\textrm{FD},1} \in \R^{d \times d}$ is defined in \eqref{eq: laplacianFD}, and $v_j^{(i)} \sim \textrm{Ber}_p\{1,2\}.$ We then check directly whether the centralizer of $\{A^{(1)}, \dots, A^{(N)}\}$ is trivial by computing the rank of the associated $Nd^2 \times d^2$-dimensional coefficient matrix. The probabilities are computed by running 300 parallel experiments for each $N$ and evaluating the proportion of instances for which the centralizer is trivial. Figure \ref{fig: 1} plots the probability that the centralizer of the resulting set of matrices is trivial as a function of the sample size $N$, for several values of $p$. When $p$ is not too small, this probability rises rapidly toward $1$ as the sample size $n$ increases. In contrast, for small $p$ the probability remains bounded away from $1$ even with many samples, because the sampled potential matrices are likely to be nearly constant (thereby violating the diversity condition); in such cases, a much larger number of samples is required to ensure that the centralizer is trivial. 

\subsubsection{In-domain generalization}\label{subsubsec: 1din-domain}
We first investigate the in-domain generalization capabilities of trained transformer models for in-context learning the family of linear systems associated to the $1D$ Schr\"odinger equation. For this, we discretize inputs according to the finite difference method described in subsection \ref{subsec: FD}, for both training and testing. The potential function $V(x)$ is defined either as a piecewise constant function with randomly-chosen values in $\{1,2\}$ (corresponding to the set-up of Section \ref{sec: applications}), or as a \textit{lognormal random field} $V(x) = e^{g(x)},$ where $g(x;\alpha,\beta) = \sum_{i=1}^{\infty} \xi_i \left(i^2\pi^2 + \alpha \right)^{-\beta/2} \sin(i \pi x)$ and $\xi_i$ are i.i.d. standard normal random variables. The numerical results presented in Figure \ref{fig: 2} show the error scales like $O(m^{-1}),$ where $m$ is the length of the prompt used in the inference procedure (i.e., the number of input-output pairs provided from the given linear system).

\begin{figure}
    \centering
    \includegraphics[width=0.5\linewidth]{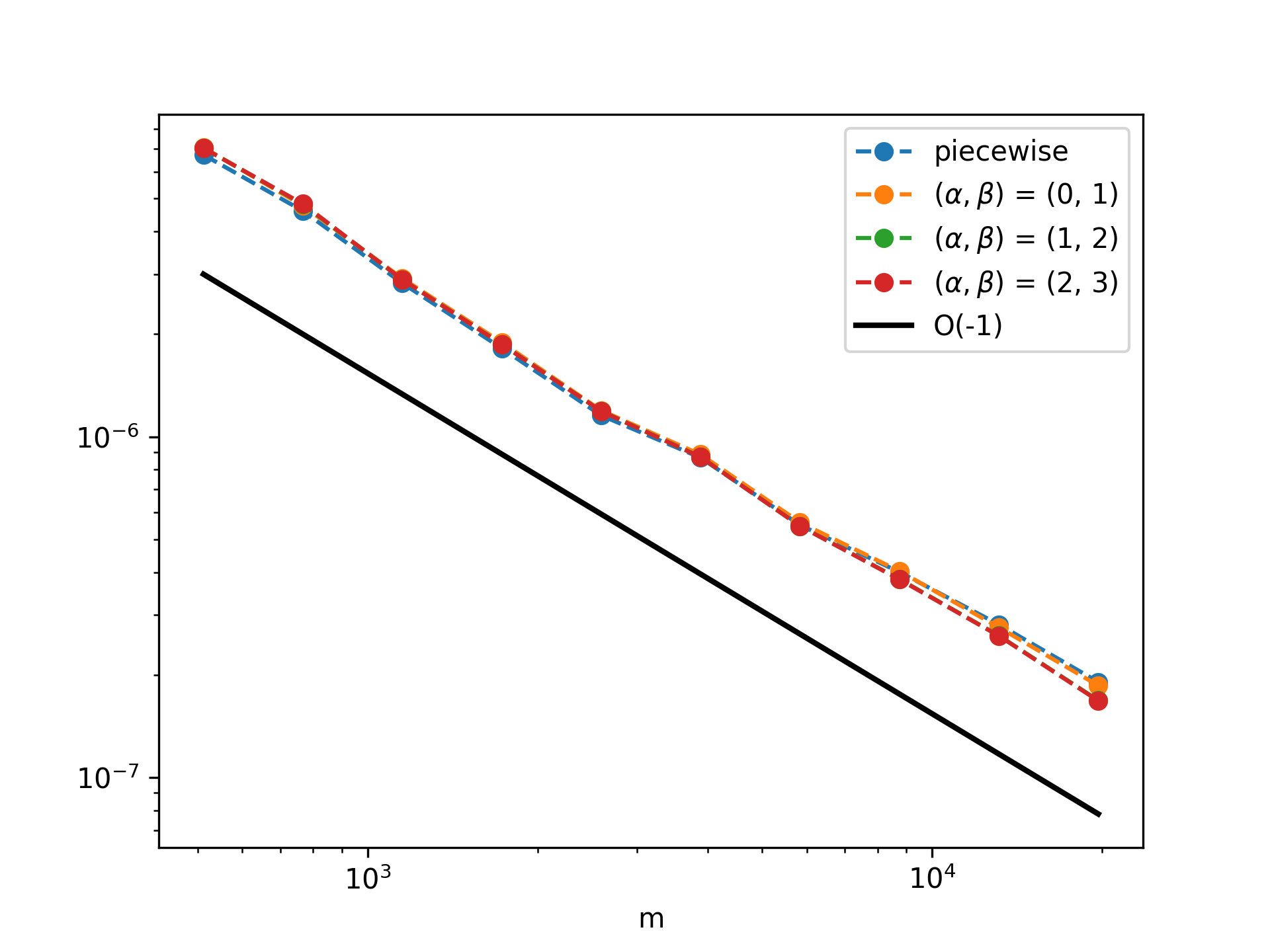}
    \caption{In-domain generalization MSE with respect to inference prompt length $m$ with $D=1, d=10, N=40000$ training tasks, $n=40000$ training prompt length, $M=1000$ testing tasks. Both training and testing are performed with the finite difference method. The four plots are overlaid with various distributions for potential function $V(x)$: piecewise constant, and lognormal with parameters $(\alpha,\beta) = (0,1),(1,2),(2,3)$. The solid black line corresponds to the $O(m^{-1})$ scaling rate.}
    \label{fig: 2}
\end{figure}

\subsubsection{Out-of-domain generalization}\label{subsubsec: 1dood}

We also investigate out-of-domain generalization capabilities. For this, we vary the discretization method and distribution of the potential function $V$. We consider both training and testing on the finite difference method and the finite element method stated in Section \ref{subsec: FEM}. We also consider training/testing on piecewise constant $V$ and testing/training on lognormal $V$. Lastly, we perform a shift in $V$ by fixing $\alpha$ and varying $\beta$, and fixing $\beta$ and varying $\alpha$. The results shown in Figure \ref{fig: 3} scale like $O(m^{-1})$, showing that the linear transformer is able to perform out-of-domain generalization.
\begin{figure}
    \centering
    \includegraphics[width=1\linewidth]{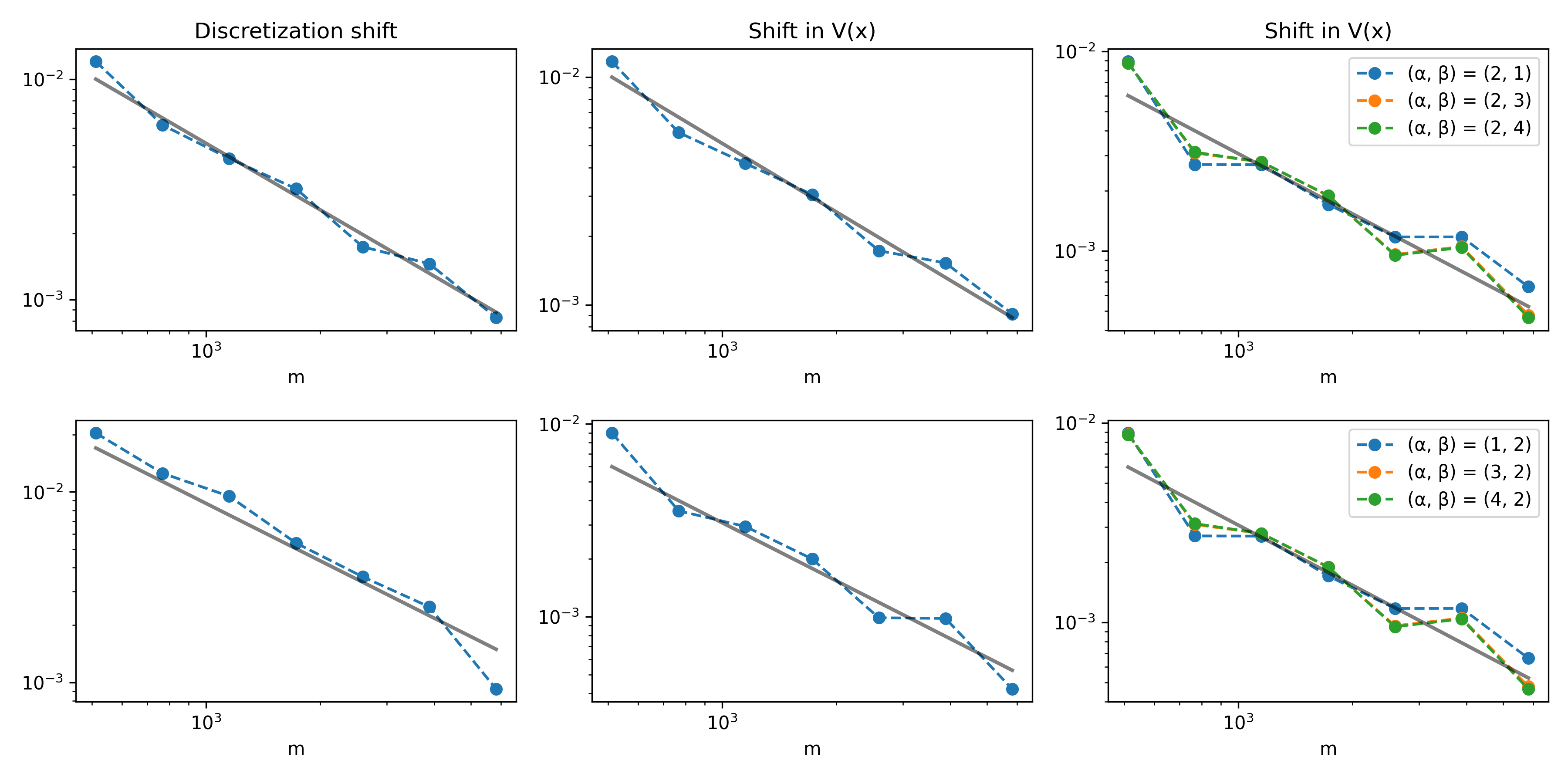}
    \caption{Out-of-domain generalization shifted relative error with respect to inference prompt length $m$ with $d=10$, $N=40000$ training tasks, $n=40000$ training prompt length, $M=10000$ testing tasks. The transparent black lines correspond to the $O(m^{-1})$ scaling rate.
    Top left: we train with FD and test with FEM.
    Bottom left: we train with FEM and test with FD.
    Top middle: we train with $V$ as a piecewise constant, and test with $V$ from a lognormal random field with $(\alpha,\beta)=(2,2)$.
    Bottom middle: we train with $V$ from a lognormal random field with $(\alpha,\beta)=(2,2)$, and test with $V$ as a piecewise constant.
    Top right: we train with $V$ from a lognormal random field with $(\alpha,\beta)=(2,2)$ and test with $V$ from a lognormal random field with varied $\beta$.
    Bottom right: we train with $V$ from a lognormal random field with $(\alpha,\beta)=(2,2)$ and test with $V$ from a lognormal random field with varied $\alpha$.}
    \label{fig: 3}
\end{figure}

\subsection{Experiments for 2D Schr\"odinger operators}\label{subsec: 2dexperiments}
For this section, we consider the same system as in Section \ref{subsec: 1dexperiments}, but with $D=2$.

\subsubsection{Verifying the diversity condition numerically}

We perform the same experiments as in Section \ref{subsubsec: 1ddiversity}, but with $D=2$. in Figure \ref{fig: 4}. We again see that for moderate values of the Bernoulli parameter $p$, the rank of the coefficient matrix converges to $d^2-1$ rapidly as $n$ increases. However, the convergence is overall slower than the $1D$ case, as predicted in Theorem \ref{thm: FD}. The linear system dimension for this experiment is $d = 9.$

\begin{figure}
    \centering
    \includegraphics[width=0.6\linewidth]{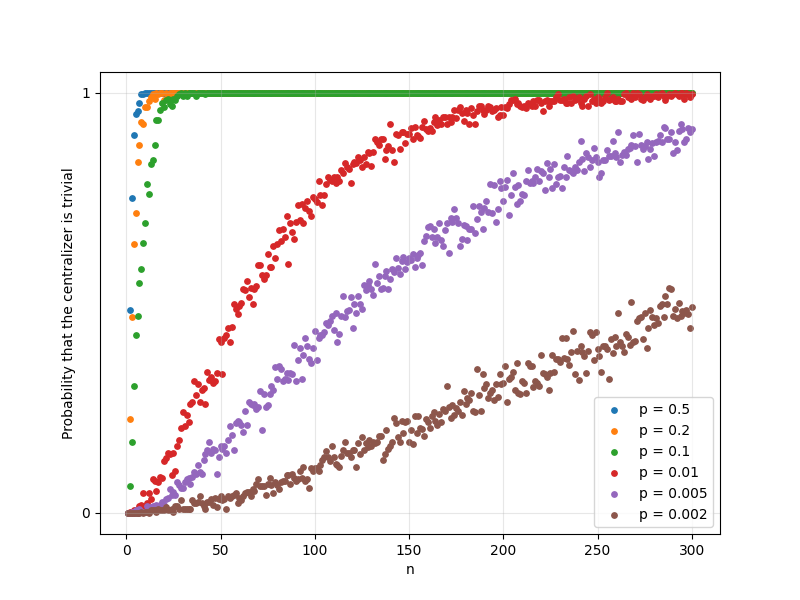}
    \caption{The probability that the centralizer of  discrete random 2D Schr\"odinger matrices—whose potential entries are sampled from ($\mathrm{Ber}_p\{1,2\}$)—is trivial, shown as a function of the sample size $N$. Different colors correspond to different values of the Bernoulli parameter $p$. The dimension of the linear systems here is $d = 9.$}
    \label{fig: 4}
\end{figure}

\subsubsection{In-domain generalization}
We perform the same numerical experiment as in Section \ref{subsubsec: 1din-domain}, but with $D=2,d=25$. The results presented in Figure \ref{fig: 5} indicate that the error again scales like $O(m^{-1})$. 

\begin{figure}
    \centering
    \includegraphics[width=0.5\linewidth]{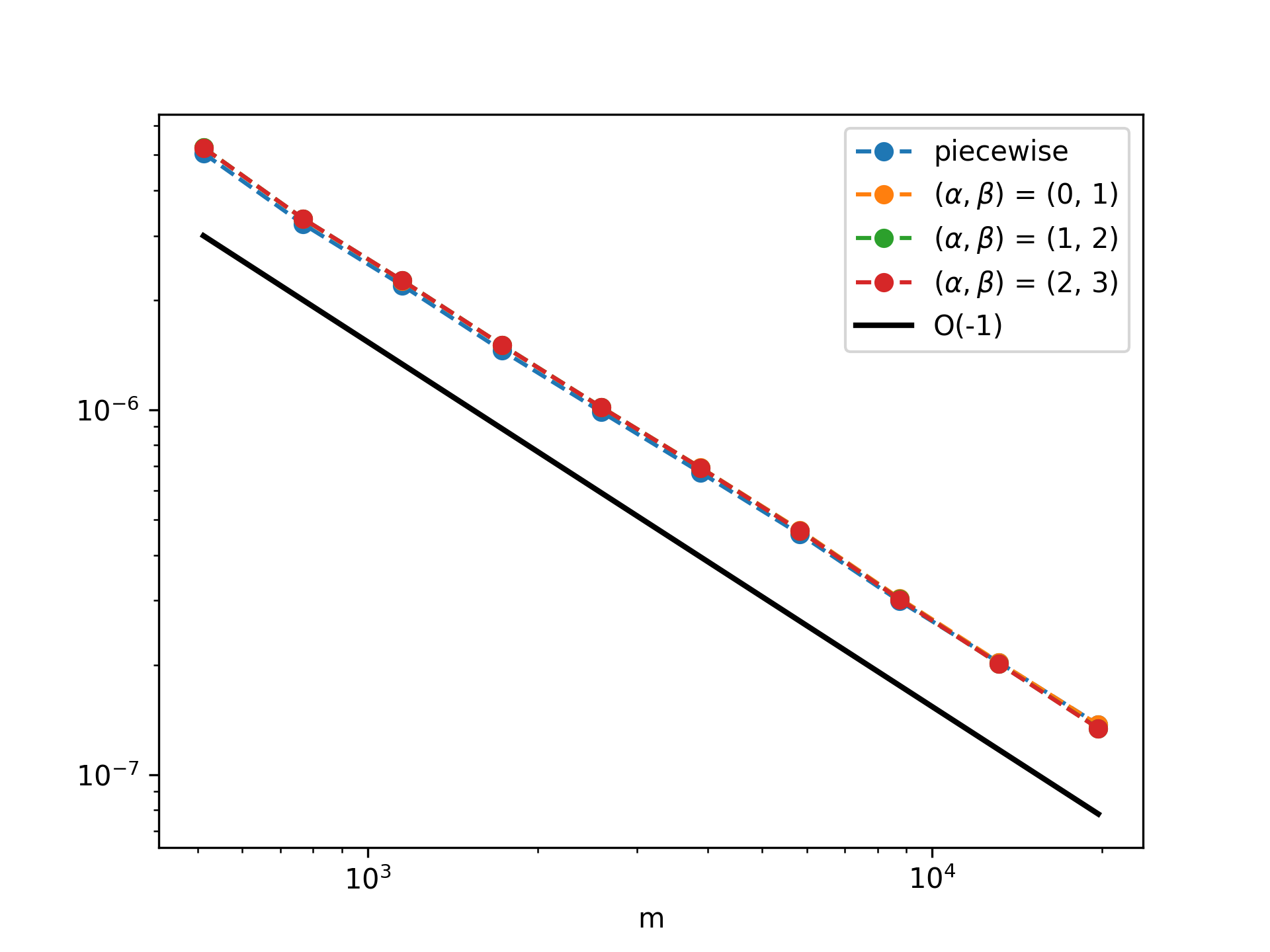}
    \caption{In-domain generalization MSE with respect to inference prompt length $m$ with $D=2, d=25, N=40000$ training tasks, $n=40000$ training prompt length, $M=1000$ testing tasks. Both training and testing is performed with the finite difference method. 4 plots are overlaid with various distributions for potential function $V(x)$: piecewise constant, and lognormal with parameters $(\alpha,\beta) = (0,1),(1,2),(2,3)$. The solid black line corresponds to the $O(-1)$ scaling rate.}
    \label{fig: 5}
\end{figure}

\subsubsection{Out-of-domain generalization}
We perform the same out-of-domain experiments as in Section \ref{subsubsec: 1dood}, but with $D=2,d=25$. The results presented in Figure \ref{fig: 6} indicate that the error again scales like $O(m^{-1})$.
\begin{figure}
    \centering
    \includegraphics[width=1\linewidth]{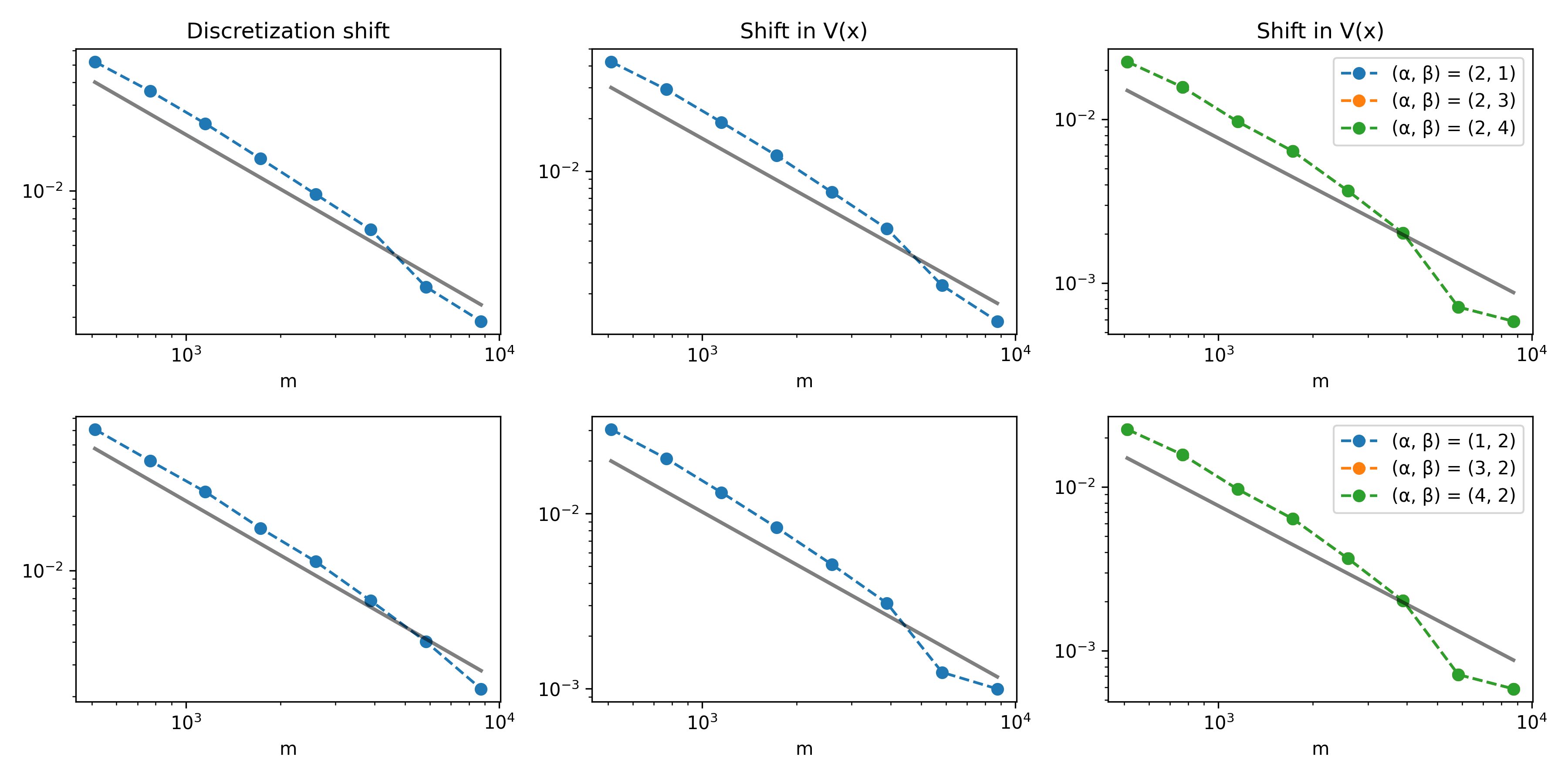}
    \caption{Out-of-domain generalization shifted relative error with respect to inference prompt length $m$ with $d=10$, $N=40000$ training tasks, $n=40000$ training prompt length, $M=10000$ testing tasks. The transparent black lines correspond to the $O(-1)$ scaling rate.
    Top left: we train with FD and test with FEM.
    Bottom left: we train with FEM and test with FD.
    Top middle: we train with $V$ as a piecewise constant, and test with $V$ from a lognormal random field with $(\alpha,\beta)=(2,2)$.
    Bottom middle: we train with $V$ from a lognormal random field with $(\alpha,\beta)=(2,2)$, and test with $V$ as a piecewise constant.
    Top right: we train with $V$ from a lognormal random field with $(\alpha,\beta)=(2,2)$ and test with $V$ from a lognormal random field with varied $\beta$.
    Bottom right: we train with $V$ from a lognormal random field with $(\alpha,\beta)=(2,2)$ and test with $V$ from a lognormal random field with varied $\alpha$.}
    \label{fig: 6}
\end{figure}

\section{Conclusion and outlook}
In this work, we introduced a mathematical notion of diversity of matrix-valued probability distributions, motivated by the generalization ability of transformer neural networks. We proved sufficient conditions on matrix distributions to ensure that the diversity condition in \eqref{prob: diversity} holds, with quantitative estimates on the required number of samples. We applied these general results to prove that matrix distributions arising from discretizations of Schr\"odinger operators with random coefficients satisfy the diversity condition in \eqref{prob: diversity}; this implies that transformers trained to solve PDEs defined by random Schr\"odinger operators possess powerful generalization properties. We demonstrated the later claim through several numerical experiments. There are several important directions for future work. First, we would like to extend our results on random Schr\"odinger operators to broader families of differential operators, such as linear elliptic operators with non-constant diffusion coefficients, or time-dependent linear differential operators. We would also like to improve the dependence on the matrix dimension $d$ of some of the sample complexity estimates in this paper. Finally, as a broader goal, we would like to explore the applications of random matrix diversity to other areas of applied mathematics. We leave these problems to future work.

\section{Additional proofs }\label{sec: proofs}
\begin{proof}[Proof of Theorem \ref{thm: FD}]
    We first prove Theorem \ref{thm: FD} in the one-dimensional case $D = 1$, after which we will generalize the result to $D > 1.$
    
    \paragraph{Proof when $D = 1$:} In dimension $D = 1$ (so the matrix dimension is $M \times M$), each sample $\mathbf{A}^{(i)}$ from $\p_{\textrm{FD}}$ takes the form $\mathbf{A}^{(i)} = -\mathbf{\Delta}_{\textrm{FD},1} + \mathbf{V}^{(i)}$, where $\mathbf{V}^{(i)} = \textrm{diag}(\bv^{(i)})$ and $\bv^{(i)} = (\bv_1^{(i)}, \dots, \bv_M^{(i)})$ with $\bv_j^{(i)} \sim \textrm{Ber}_p\{a,b\}.$ We will use the fact that the discrete Laplace matrix $\mathbf{\Delta}_{\textrm{FD},1}$ has distinct eigenvalues and can be diagonalized by the cosine basis \cite{gray2006toeplitz}. In more detail, if we define $\mathbf{u}^{(k)} \propto \left( 1, \textrm{cos} \left(\frac{2\pi k}{M} \right), \dots, \cos \left(\frac{2\pi(d-1)k}{M} \right) \right)$ and $\mathbf{U} = \begin{bmatrix}
        \mathbf{u}_1 & \dots & \mathbf{u}_M
    \end{bmatrix},$ then $\mathbf{U}^T \mathbf{\Delta}_{\textrm{FD},1} \mathbf{U}$ is diagonal. Thus, to verify Assumption \ref{assum: main}, we need only derive an upper bound $\p_{\textrm{FD}} \left(\mathbf{u}_k^T \mathbf{V} \mathbf{u}_{k+1} = 0 \right),$ uniformly in $k \in [d-1].$ Since $\mathbf{V} = \textrm{diag}(\bv)$, we can write $\mathbf{u}_k^T \mathbf{V} \mathbf{u}_{k+1} = \langle \mathbf{w}^{(k)}, \bv \rangle$, where $\mathbf{w}^{(k)} \in \R^d$ has entries 
    \begin{equation}\label{eq: wk}
        \mathbf{w}_j^{(k)} = \cos \left(\frac{2\pi k(j-1)}{M} \right) \cos \left(\frac{2\pi(k+1)(j-1)}{M} \right).
    \end{equation}
    Thus our goal is to prove an upper bound on $\p_{\textrm{FD}}(\langle \mathbf{w}^{(k)}, \bv \rangle = 0)$, where $\mathbf{v}$ is a random vector with independent coordinates sampled from $\textrm{Ber}_p(\{a,b\}.$ Corollary 5.5 in \cite{madiman2023bernoulli} establishes an inequality of the form 
    \begin{equation}\label{eq: bernoulliprob}
        \sup_{a \in \R} \mathbb{P} \left( \langle \bv, w \rangle = a \right) \leq \frac{1}{\sqrt{1+2\|w\|_0 p(1-p)},}
    \end{equation}
    where $\|\mathbf{w}_0\|$ is the number of non-zero entries of $w.$ To apply this result in our setting, we need to bound $\|\mathbf{w}^{(k)}\|_0$ from below, uniformly in $k.$ We claim that $\|\mathbf{w}^{(k)}\|_0$ can be bounded from below by a constant proportional to $d$, and the proof involves only a simple number-theoretic argument.



    \begin{claim}\label{claim2}
        For each $k \in [d-1]$, let $\mathbf{w}^{(k)} \in \R^d$ be defined by Equation \eqref{eq: wk}. Then $\|\mathbf{w}^{(k)}\|_0 \geq \frac{M}{3}.$
    \end{claim}
    Note that from Claim \ref{claim2}, we can deduce that $\p_{\textrm{FD}}(\mathbf{u}_k^T \mathbf{V} \mathbf{u}_{k+1}) \leq \frac{1}{\sqrt{1+\frac{2}{3} d p (1-p)}},$ and hence by Theorem \ref{thm: main}, we are guaranteed that $\{\mathbf{A}^{(1)}, \dots, \mathbf{A}^{(N)}, -\mathbf{\Delta}_{\textrm{FD},1}\}$ has a trivial centralizer with probability at least $1 - (d-1) \left(1+\frac{2}{3} d p (1-p)\right)^{-N/2}.$

    \paragraph{Proof of Claim \ref{claim2}:} We aim to prove that, for each $k \in [d-1],$ the vector $\mathbf{w}^{(k)}$ given by Equation \eqref{eq: wk} satisfies $\|w\|_0 \geq \frac{M}{3}.$ Let us fix $k \in [d].$ For $j \in [d],$ we see that $\mathbf{w}_j^{(k)} = 0$ if and only if $\cos\left( \frac{2\pi k(j-1)}{M}\right) = 0$ or $\cos \left( \frac{2\pi(k+1)(j-1)}{M} \right)  = 0.$ For a fixed integer $\ell$, $\cos\left( \frac{2 \pi \ell}{M} \right) = 0$ if and only if $d$ is a multiple of $4$ and $2\ell$ is congruent to $d/4$ or $3/4$ modulo $d$. Equivalently, we must have $d = 4m$ for some integer $m \in \mathbb{N}$ and $\ell$ must be of the form $\ell = \frac{m}{2} + nm$ or $\ell = \frac{3m}{2} + nm$, where $n$ ranges over the positive integers. Since $\ell$ is a positive integer, this clearly forces $m$ to be even, say $m = 2z$ for $z \in \mathbb{N}$. It is easy to see that both congruence equations imply that $\ell$ must be an odd multiple of $z,$ so we write $\ell = (2q + 1)z$ for $q \in \mathbb{N}$. In our case, the integer $\ell$ is also of the form $\ell = k(j-1)$ or $\ell = (k+1)(j-1)$. In each case, we show that only a constant fraction of possible choices of $j \in [8z]$ will satisfy this equation. Let us first suppose that $\ell = k(j-1)$ and seek solutions to $$k(j-1) = (2q + 1)z.$$
    
    If $z = k$, then the equation is solvable whenever $j \in [d]$ is even, and hence there are $\lfloor \frac{M}{2} \rfloor$ possible choices for $j$. We can also clearly assume that $j-1 \neq 0.$ We can therefore assume that $z \neq k$. Let $g = \gcd(k,z)$ be the greatest common divisor of $k$ and $z$, and define $L := k/g$ and $B := z/g$ with $B, L > 1.$ Then, dividing the above equation by $g$, we get
    \begin{align*}
        L(j-1) = (2q+1)M \in \mathbb{N}.
    \end{align*}
    In particular, this implies that $(j-1)$ is a multiple of $B$ (if $L$ is a multiple of $B$, then $Bg$ is a divisor of $k$, which contradicts that maximality of $g$). Since $j \in [8z]$, this constrains $j-1$ to be of the form $j-1 = cB$, where $c \in \{1, 2, \dots, 8g-1\}$. Next, we will show that $j-1$ must be odd. To show this, recall that $\nu_2(\cdot)$ is the 2-adic valuation, defined on positive integers as $\nu_2(a) = \max\{s: a = 2^s b, \; \textrm{$b$ odd}\}$, and extended to positive rational numbers by $\nu(a/b) = \nu_2(a) - \nu_2(b).$ In particular, the equation $k(j-1)=(2q+1)z$ implies that
    \begin{align*}
        \nu(j-1) = \nu_2(z) - \nu_2(k).
    \end{align*}
    On the other hand, since $j-1$ is a multiple of $B$, we can write it as $j-1 = cB$ for some $c \in \mathbb{N}$, which tells us
    \begin{align*}
        \nu_2(j-1) &= \nu_2(c) + \nu_2(B) \\
        &= \nu_2(c) + \nu_2(z) - \nu_2(g) \\
        &= \nu_2(c) + \nu_2(z) - \min(\nu_2(k),\nu_2(z)),
    \end{align*}
    where $\nu_2(g) = \min(\nu_2(k),\nu_2(z))$ follows from the fact that $g$ is the GCD of $k$ and $z$. Taking the difference of these two expressions for $\nu_2(j-1)$, we find that
    \begin{align*}
        \nu_2(c) + \nu_2(k) - \min(\nu_2(k),\nu_2(z)) = 0,
    \end{align*}
    which implies that $\nu_2(c) = 0$, and hence $c$ is odd. Therefore, $j-1$ must be of the form $cM$, where $c \in [8g-1]$ is odd, and there are at most $4g$ such elements. Hence, the number of positive integer solutions $(j,q)$ with $j \in [8z]$ to the equation $k(j-1) = (2q+1)z$ is at most $4g$, where $g$ is the $\gcd$ of $k$ and $z$ (note that when $k = z$, $g = z$ and $4g = d/2$, which recovers the edge case).

    For the case $\ell = (k+1)(j-1)$, the above analysis gives the same upper bound. By a union bound,
    \begin{align*}
        \# \{j \in [d]: \mathbf{w}^{(k)}_j = 0\} \leq \begin{cases}
            4 (\gcd(k,z) + \gcd(k+1,z)), \; \textrm{if $d = 8z$, $z \in \mathbb{N}$} \\
            0, \; \textrm{otherwise.}
        \end{cases}
    \end{align*}
    To conclude, we first note that for any $k \in \mathbb{N}$, $\gcd(k,z) \cdot \gcd(k+1,z) \leq z$; this follows from the fact that $\gcd(k,z)$ and $\gcd(k,k+1)$ are relatively prime (to see this, note that if $y \in \mathbb{N}$ divides $\gcd(k,z)$ and $\gcd(k+1,z)$, then $y$ divides both $k$ and $k+1$, which are relatively prime for any $k \in \mathbb{N}$) and each one divides $z$. Applying the inequality $x+y \leq xy + 1$, valid for any pair of non-negative integers $(x,y)$, we see that $\gcd(k,z) + \gcd(k+1,z) \leq z+1$, and hence $4 (\gcd(k,z) + \gcd(k+1,z)) \leq 4z + 1 \leq \frac{M}{2} + 1 \leq \lfloor \frac{2d}{3} \rfloor$ if $d \geq 6.$ We conclude that
     \begin{align*}
        \# \{j \in [d]: \mathbf{w}^{(k)}_j = 0\} \leq \begin{cases}
            \frac{2d}{3}, \; \textrm{if $d = 8z$, $z \in \mathbb{N}$} \\
            0, \; \textrm{otherwise.}
        \end{cases}
    \end{align*}
    Having proven Claim \ref{claim2}, we have established Theorem \ref{thm: FD} when $D = 1.$

    \paragraph{Extension to $D > 1$:} Having proven Theorem \ref{thm: FD} when $D = 1$, we generalize to $D > 1$ by induction on $D$. Specifically, we will show that for each $D \in \mathbb{N}$, there is a recursively-defined constant $c_D \in (0,1)$ such that $\p_{\textrm{FD}}$ satisfies Assumption \ref{assum: main} with constant $c_D$. Then, using that $c_1 = \frac{1}{\sqrt{1+\frac{2}{3}M p(p-1)}}$ and the relationship between $c_{D-1}$ and $c_D$, we will establish an explicit upper bound on $c_D.$
    
    Before proceeding with the proof, we need to formalize several notations. First, let us recall that the orthogonal matrix $\mathbf{U} \in \R^d$, whose $k^{\textrm{th}}$ column is proportional to $\left( 1, \textrm{cos} \left(\frac{2\pi k}{M} \right), \dots, \cos \left(\frac{2\pi(d-1)k}{M} \right) \right),$ diagonalizes the matrix $\mathbf{\Delta}_{\textrm{FD},1}.$ This fact, combined with Equation \eqref{eq: higherlaplacian} for the matrix representation $\mathbf{\Delta}_{\textrm{FD},D}$ of the $D$-dimensional Laplace operator, shows that the eigenvectors of $\mathbf{\Delta}_{\textrm{FD},D}$ are given by $\{\mathbf{u}_{k_1} \otimes \dots \otimes \mathbf{u}_{k_D}\}_{(k_1, \dots, k_D) \in [d]^D}.$ In addition, since $\mathbf{\Delta}_{\textrm{FD},1}$ has distinct eigenvalues, Equation \eqref{eq: higherlaplacian} shows that $\mathbf{\Delta}_{\textrm{FD},k}$ has distinct eigenvalues for any $k \in \mathbb{N}.$ Based on the enumeration $\{\mathbf{u}_1, \dots, \mathbf{u}_M\}$ of the eigenvectors of $\mathbf{\Delta}_{\textrm{FD},1}$, we construct an enumeration of the eigenvectors of $\mathbf{\Delta}_{\textrm{FD},k}$ for each $k \in \mathbb{N}$. We define this enumeration inductively: assuming $k > 1$ and that we have enumerated the eigenvectors of $\mathbf{\Delta}_{\textrm{FD},d}$ as $\{\mathbf{u}_1^{(k-1)}, \dots, \mathbf{u}_{M^{k-1}}^{(k-1)}\} \subset \R^{M^{k-1}},$ we enumerate the eigenvectors of $\mathbf{\Delta}_{\textrm{FD},k}$ as
     \begin{align}\label{eq: eigenvectorenumeration}
        \left\{ \mathbf{u}^{(k-1)}_1 \otimes \mathbf{u}_1, \dots, \mathbf{u}^{(k-1)}_{M^{k-1}} \otimes \mathbf{u}_1, \mathbf{u}^{(k-1)}_1 \otimes \mathbf{u}_2, \dots, \mathbf{u}^{(k-1)}_{M^{k-1}} \otimes \mathbf{u}_2, \dots, \mathbf{u}^{(k-1)}_1 \otimes \mathbf{u}_M, \dots, \mathbf{u}^{(k-1)}_{M^{k-1}} \otimes \mathbf{u}_M \right\},
    \end{align}
    and we denote this enumeration as $\{\mathbf{u}_1^{(k)}, \dots, \mathbf{u}_{M^k}^{(k)}\} \subset \R^{M^k}.$

    Next, we describe the matrix representation corresponding to the potential function defined in Equation \eqref{eq: separablepotential}. For $k \in \mathbb{N}$, let us write $\mathbf{V}_k = \bigotimes_{j=1}^{k} \mathbf{V}_k^{(j)}$, where $\mathbf{V}_k^{(1)}, \dots, \mathbf{V}_k^{(k)} \in \R^d$ are independent random diagonal matrices whose entries are independently sampled from $\textrm{Ber}_p\{a,b\}.$ Note that if $\mathbf{V}_{D,1}, \dots, \mathbf{V}_{D,m}$ are iid copies of the matrix $\mathbf{V}_D$, then the sum $\sum_{i=1}^{m} \mathbf{V}_{D,i}$ gives the matrix representation of the separable $D$-dimensional potential function given in Equation \eqref{eq: separablepotential}. In addition, for $j,k \in \mathbb{N},$ $\mathbf{V}_j \otimes \mathbf{V}_k$ and $\mathbf{V}_{j+k}$ have the same distribution. 

    We now proceed with the inductive proof. Let $D > 1$ and suppose that the enumeration of the eigenvectors of $\mathbf{\Delta}_{\textrm{FD},D-1}$ defined by \eqref{eq: eigenvectorenumeration} satisfies
    \begin{align*}
        \sup_{a \in \R} \p_{\textrm{FD}} \left((\mathbf{u}^{(D-1)}_{k})^T \mathbf{V}_{D-1} \mathbf{u}_{k+1}^{(D-1)} = a \right) \leq c_{D-1}
    \end{align*}
    for some constant $c_{D-1}$ Note we have already established the above inequality for the case $D = 1$ with $c_1 = \frac{1}{\sqrt{1+\frac{2}{3} d p (1-p)}}$ for some constant $C > 0.$ We aim to show that the corresponding enumeration $\{\mathbf{u}_1^{(D)}, \dots, \mathbf{u}_{M^D}^{(D)}\}$ of eigenvectors of $\mathbf{\Delta}_{\textrm{FD},D}$ satisfies the condition of Assumption \ref{assum: main} for another constant $c_D$, i.e. that for each $k \in [d^D -1],$
    \begin{align}\label{eq: inductivestep}
        \sup_{a \in \R} \p_{\textrm{FD}} \left( (\mathbf{u}_k^{(D)})^T \left( \sum_{i=1}^{m} \mathbf{V}_{D,i} \right) \mathbf{u}_{k+1} = a \right) \leq c_D.
    \end{align}
    To prove this, we first observe that, for any $k \in [d^D-1]$ and $a \in \R$
    \begin{align*}
        \p_{\textrm{FD}} \left( (\mathbf{u}_k^{(D)})^T \left( \sum_{i=1}^{m} \mathbf{V}_{D,i} \right) \mathbf{u}_{k+1}^{(D)} = a \right) &= \p_{\textrm{FD}} \left((\mathbf{u}_k^{(D)})^T \mathbf{V}_{D,1} \mathbf{u}_{k+1}^{(D)} = a - \sum_{i=2}^{m} (\mathbf{u}_k^{(D)})^T \mathbf{V}_{D,i} \mathbf{u}_{k+1}^{(D)} \right) \\
        &\leq \sup_{a' \in \R} \p_{\textrm{FD}} \left( (\mathbf{u}_k^{(D)})^T \mathbf{V}_{D,1} \mathbf{u}_{k+1}^{(D)} = a' \right).
    \end{align*}
    This shows that to verify Inequality \eqref{eq: inductivestep} it suffices to bound $\sup_{a \in \R} \p_{\textrm{FD}} \left( (\mathbf{u}_k^{(D)})^T \mathbf{V}_{D,1} \mathbf{u}_{k+1}^{(D)} = a \right)$. In other words, despite that the potential function is given by a sum of iid separable functions, it suffices to verify Assumption \ref{assum: main} for the matrix representation of a single term in the sum. To bound $\sup_{a \in \R} \p_{\textrm{FD}} \left( (\mathbf{u}_k^{(D)})^T \mathbf{V}_{D,1} \mathbf{u}_{k+1}^{(D)} = a \right),$ we consider two cases. First, if $\mathbf{u}_k^{(D)} = \mathbf{u}_j^{(D-1)} \otimes \mathbf{u}_{\ell}$ for some $j \in [d^{(D-1)}-1]$ and $\ell \in [d]$, then $\mathbf{u}_{k+1} = \mathbf{u}_{j+1}^{(D-1)} \otimes \mathbf{u}_{\ell},$ and for any $a \in \R,$ we have 
    \begin{align*}
        \p_{\textrm{FD}} \left((\mathbf{u}_j^{(D-1)} \otimes \mathbf{u}_{\ell})^T \mathbf{V}_D (\mathbf{u}_{j+1}^{(D-1)} \otimes \mathbf{u}_{\ell}) = a \right) &= \p_{\textrm{FD}} \left((\mathbf{u}_j^{(D-1)} \otimes \mathbf{u}_{\ell})^T (\mathbf{V}_{D-1} \otimes \mathbf{V}_1) \otimes (\mathbf{u}_{j+1}^{(D-1)} \otimes \mathbf{u}_{\ell}) = 0 \right) \\
        &= \p_{\textrm{FD}} \left( (\mathbf{u}_j^{(D-1)})^T \mathbf{V}_{D-1} \mathbf{u}_{j+1}^{(D-1)} \cdot \mathbf{u}_{\ell}^T \mathbf{V}_1 \mathbf{u}_{\ell} = a \right).
    \end{align*}
    Since $\mathbf{u}_{\ell}^T \mathbf{V}_1 \mathbf{u}_{\ell}$ is never equal to zero, we have
    \begin{align*}
        \p_{\textrm{FD}} \left((\mathbf{u}_j^{(D-1)} \otimes \mathbf{u}_{\ell})^T \mathbf{V}_D (\mathbf{u}_{j+1}^{(D-1)} \otimes \mathbf{u}_{\ell}) = a \right) &= \p_{\textrm{FD}} \left( (\mathbf{u}_j^{(D-1)})^T \mathbf{V}_{D-1} \mathbf{u}_{j+1}^{(D-1)} = \frac{a}{\mathbf{u}_{\ell}^T \mathbf{V}_1 \mathbf{u}_{\ell}} \right) \\
        &\leq \sup_{a' \in \R} \p_{\textrm{FD}} \left( (\mathbf{u}_j^{(D-1)})^T \mathbf{V}_{D-1} \mathbf{u}_{j+1}^{(D-1)} = a' \right) \\
        &\leq c_{D-1}
    \end{align*}
    where we used the induction hypothesis in the last step. This proves that
    \begin{align*}
        \sup_{a \in \R} \p_{\textrm{FD}} \left( (\mathbf{u}_k^{(D)})^T \mathbf{V}_D \mathbf{u}_{k+1}^{(D)} = a \right) \leq c_{D-1}
    \end{align*}
    in the first case. For the second case, suppose that $\mathbf{u}_k = \mathbf{u}_{M^{(D-1)}}^{(D-1)} \otimes \mathbf{u}_{\ell}$ where $\ell \in [d-1].$ Then $\mathbf{u}_{k+1} = \mathbf{u}_1^{(D-1)} \otimes \mathbf{u}_{\ell+1}$ and, for any $a \in \R,$ we have
    \begin{align*}
        &\p_{\textrm{FD}} \left( (\mathbf{u}_{M^{(D-1)}}^{(D-1)} \otimes \mathbf{u}_{\ell})^T \mathbf{V}_D (\mathbf{u}_1^{(D-1)} \otimes \mathbf{u}_{\ell+1}) = a \right) = \p_{\textrm{FD}} \left( (\mathbf{u}_{M^{(D-1)}}^{(D-1)})^T \mathbf{V}_{D-1} \mathbf{u}_1^{(D-1)} \cdot \mathbf{u}_{\ell}^T \mathbf{V}_1 \mathbf{u}_{\ell+1} = a \right) \\
        &= \p_{\textrm{FD}}\left((\mathbf{u}_{M^{(D-1)}}^{(D-1)})^T \mathbf{V}_{D-1} \mathbf{u}_1^{(D-1)} = 0 \right) \cdot \p_{\textrm{FD}} \left( (\mathbf{u}_{M^{(D-1)}}^{(D-1)})^T \mathbf{V}_{D-1} \mathbf{u}_1^{(D-1)} \cdot \mathbf{u}_{\ell}^T \mathbf{V}_1 \mathbf{u}_{\ell+1} = a \big| (\mathbf{u}_{M^{(D-1)}}^{(D-1)})^T \mathbf{V}_{D-1} \mathbf{u}_1^{(D-1)} = 0 \right) \\
        &+ \left( 1- \p_{\textrm{FD}}\left((\mathbf{u}_{M^{(D-1)}}^{(D-1)})^T \mathbf{V}_{D-1} \mathbf{u}_1^{(D-1)} = 0 \right)\right) \cdot \p_{\textrm{FD}} \left( (\mathbf{u}_{M^{(D-1)}}^{(D-1)})^T \mathbf{V}_{D-1} \mathbf{u}_1^{(D-1)} \cdot \mathbf{u}_{\ell}^T \mathbf{V}_1 \mathbf{u}_{\ell+1} = a \big| (\mathbf{u}_{M^{(D-1)}}^{(D-1)})^T \mathbf{V}_{D-1} \mathbf{u}_1^{(D-1)} \neq 0 \right) \\
        &\leq \p_{\textrm{FD}}\left((\mathbf{u}_{M^{(D-1)}}^{(D-1)})^T \mathbf{V}_{D-1} \mathbf{u}_1^{(D-1)} = 0 \right) \cdot \mathbf{1}_{x=0}(a) + \left( 1- \p_{\textrm{FD}}\left((\mathbf{u}_{M^{(D-1)}}^{(D-1)})^T \mathbf{V}_{D-1} \mathbf{u}_1^{(D-1)} = 0 \right)\right) \cdot \sup_{a' \in \R} \p_{\textrm{FD}} \left( \mathbf{u}_{\ell}^T \mathbf{V}_1 \mathbf{u}_{\ell+1} = a' \right) \\
        &\leq \p_{\textrm{FD}}\left((\mathbf{u}_{M^{(D-1)}}^{(D-1)})^T \mathbf{V}_{D-1} \mathbf{u}_1^{(D-1)} = 0 \right) + \left( 1- \p_{\textrm{FD}}\left((\mathbf{u}_{M^{(D-1)}}^{(D-1)})^T \mathbf{V}_{D-1} \mathbf{u}_1^{(D-1)} = 0 \right)\right) \cdot \sup_{a' \in \R} \p_{\textrm{FD}} \left( \mathbf{u}_{\ell}^T \mathbf{V}_1 \mathbf{u}_{\ell+1} = a' \right) \\
        &\leq \p_{\textrm{FD}}\left((\mathbf{u}_{M^{(D-1)}}^{(D-1)})^T \mathbf{V}_{D-1} \mathbf{u}_1^{(D-1)} = 0 \right) + \left( 1- \p_{\textrm{FD}}\left((\mathbf{u}_{M^{(D-1)}}^{(D-1)})^T \mathbf{V}_{D-1} \mathbf{u}_1^{(D-1)} = 0 \right)\right) \cdot \frac{1}{\sqrt{1+\frac{2}{3} d p (1-p)}},
    \end{align*}
    where we applied the inequality \eqref{eq: bernoulliprob} in the last line. Our goal is therefore to upper bound $\p_{\textrm{FD}}\left((\mathbf{u}_{M^{(D-1)}}^{(D-1)})^T \mathbf{V}_{D-1} \mathbf{u}_1^{(D-1)} = 0 \right)$
    Note that, by the construction of the enumeration $\{\mathbf{u}_1^{(D-1)}, \dots, \mathbf{u}_{M^{D-1}}^{(D-1)}\},$ we have $\mathbf{u}_1^{(D-1)} = \bigotimes_{i=1}^{D-1} \mathbf{u}_1$ and $\mathbf{u}_{M^{D-1}}^{(D-1)} = \bigotimes_{i=1}^{D-1} \mathbf{u}_M.$ We can therefore compute
    \begin{align*}
        (\mathbf{u}_{M^{(D-1)}}^{(D-1)})^T \mathbf{V}_{D-1} \mathbf{u}_1^{(D-1)} &= \prod_{i=1}^{D-1} \mathbf{u}_M^T \mathbf{V}_{D-1}^{(i)} \mathbf{u}_1 \\
        &= \prod_{i=1}^{D-1} \left( \sum_{j=1}^{M} \mathbf{u}_{M,j} \mathbf{u}_{1,j} \bv_{D-1,j}^{(i)} \right),
    \end{align*}
    where we have defined the notation $\mathbf{V}_{D-1}^{(i)} = \textrm{diag}(\bv_{D-1,1}^{(i)}, \dots, \bv_{D-1,d}^{(i)}).$ Recall from the definition of the eigenvectors $\mathbf{u}_1, \dots, \mathbf{u}_M$ that $\mathbf{u}_M$ is proportional to the vector of $1$s. We can therefore write
    \begin{align*}
         (\mathbf{u}_{M^{(D-1)}}^{(D-1)})^T \mathbf{V}_{D-1} \mathbf{u}_1^{(D-1)} = \prod_{i=1}^{D-1} \frac{\langle \bv_{D-1}^{(i)}, \mathbf{u}_1 \rangle}{C},
    \end{align*}
    for some constant $C > 0.$ Recall also that $\mathbf{u}_1$ is proportional to the vector whose $j^{\textrm{th}}$ component is given by $\cos \left( \frac{2 \pi (j-1)}{M} \right).$ A simple combinatorial argument shows that at most two of the entries of $\mathbf{u}_1$ can be zero, and therefore we have by \eqref{eq: bernoulliprob} that $\p_{\textrm{FD}}(\langle \bv_{D-1}^{(i)}, \mathbf{u}_1 \rangle = 0) \leq \frac{1}{\sqrt{1-2(M-2)p(1-p)}} $ It follows that
    \begin{align*}
        \p_{\textrm{FD}}\left((\mathbf{u}_{M^{(D-1)}}^{(D-1)})^T \mathbf{V}_{D-1} \mathbf{u}_1^{(D-1)} = 0 \right) &= \p_{\textrm{FD}}\left(\prod_{i=1}^{D-1} \frac{\langle \bv_{D-1}^{(i)}, \mathbf{u}_1 \rangle}{C} = 0 \right) \\
        &= \p_{\textrm{FD}} \left( \bigcap_{i=1}^{D-1} \{\langle \bv_{D-1}^{(i)}, \mathbf{u}_1 \rangle = 0\} \right) \\
        &= \prod_{i=1}^{D-1} \p_{\textrm{FD}} \left( \langle \bv_{D-1}^{(i)}, \mathbf{u}_1 \rangle = 0 \right) \\
        &\leq \left(  \frac{1}{\sqrt{1+2(d-2) p (1-p)}} \right)^{D-1}.
    \end{align*}
    We have therefore established that 
    \begin{align*}
        \sup_{a \in \R} \p_{\textrm{FD}} \left( (\mathbf{u}_k^{(D)})^T \mathbf{V}_D \mathbf{u}_{k+1}^{(D)} = a \right) \leq \left(  \frac{1}{\sqrt{1+2(d-2) p (1-p)}} \right)^{D-1} + \frac{1}{\sqrt{1+\frac{2}{3}d p (1-p)}}
    \end{align*}
    in the second of two cases. These cases exhaust all possible values of $k$, and we conclude that
    \begin{align*}
         \sup_{a \in \R} \p_{\textrm{FD}} \left( (\mathbf{u}_k^{(D)})^T \mathbf{V}_D \mathbf{u}_{k+1}^{(D)} = a \right) \leq c_D
    \end{align*}
    where $c_D$ is defined recursively as
    \begin{align*}
        c_D := \max \left(c_{D-1}, \left(  \frac{1}{\sqrt{1+2(d-2) p (1-p)}} \right)^{D-1} + \frac{1}{\sqrt{1+\frac{2}{3}d p (1-p)}} \right), \; \textrm{with} \; c_1 = \frac{1}{\sqrt{1+\frac{2}{3}d p (1-p)}}.
    \end{align*}
    where $C > 0$ is a universal constant. In particular, if $d \geq 3$ (so that $d-2 \geq \frac{M}{3})$, we have
    \begin{align*}
        c_D \leq \frac{2}{1+\frac{2}{3}dp(1-p)}, \; \forall D \geq 1.
    \end{align*}
    If $d > \frac{9}{2p(1-p)},$ the right-hand side above is strictly less than 1. We conclude the proof.
    
\end{proof}

\begin{proof}[Proof of Theorem \ref{thm: FD2}]
    Recall that when $D = 1$, a sample $A \sim \p_{\textrm{FD}}$ is of the form $\mathbf{A} = -\mathbf{\Delta}_{\textrm{FD,1}} + \mathbf{V},$ where $\mathbf{V} = \textrm{diag}(\bv_1, \dots, \bv_d)$ has iid entries from $\textrm{Ber}_p\{a,b\}.$ Note that the off-diagonal entries of $\mathbf{\Delta}_{\textrm{FD},1}$ are nonzero, hence it satisfies Assumption \ref{assum: 2}, item 1, with the permutation $\pi$ given by $\pi(k) = k+1$ for $k \in [d-1]$ and $\pi(d) = 1.$ In addition, since the Bernoulli distribution with probability $p$ satisfies
    \begin{align*}
        \p_{\bv_1, \dots, \bv_1, \bv_3, \bv_4 \sim \textrm{Ber}\{a,b\}} \left( \{\bv_1-\bv_2 = \bv_3-\bv_4\} \cup \{\bv_1 - \bv_2 = 0\} \cup \{\bv_3 - \bv_4 = 0\} \right) = 1-2p^2(1-p)^2,
    \end{align*}
    it follows that $\p_{\textrm{FD}}$ satisfies Assumption \ref{assum: 2}, item 2 with constant $c = 1-2p^2(1-p)^2.$ Thus, when $D = 1$ (and hence $d = M$), we have as a direct consequence of Theorem \ref{thm: 2} that
    \begin{align*}
        \p_{\textrm{FD}} \left(\textrm{$\{\mathbf{A}^{(1)}, \dots, \mathbf{A}^{(N)}\}$ has trivial centralizer} \right) \geq 1 - M(M-1)\left(1-2p^2(1-p)^2 \right)^{N/2}.
    \end{align*}
    When $D > 1$, the random part of the distribution $\p_{\textrm{FD}}$ is given by the matrix $\mathbf{V} = \mathbf{V}_1 \otimes \dots \otimes \mathbf{V}_D,$ where $\mathbf{V}_1, \dots, \mathbf{V}_D$ are independent $d \times d$ diagonal matrices with independent entries from $\textrm{Ber}_p\{a,b\}.$ Therefore, we cannot directly apply Theorem \ref{thm: 2} in this case because the entries of $\mathbf{V}$ are not independent. Instead, we generalize the proof of Theorem \ref{thm: 2} to account for this. 

    Suppose that $\mathbf{A}^{(1)}, \dots, \mathbf{A}^{(N)}$ are iid samples from $\p_{\textrm{FD}}$, and write $\mathbf{A}^{(i)} = -\mathbf{\Delta}_{\textrm{FD},D} + \left(\mathbf{V}_1^{(i)} \otimes \dots \otimes \mathbf{V}_D^{(i)} \right),$ with $\mathbf{V}_1, \dots, \mathbf{V}_D$ as defined above. Let us abuse notation and denote by $\bv_1^{(i)}, \dots, \bv_M^{(i)}$ the entries of $\mathbf{V}_1^{(i)}.$ Define the events $\{E_{j,k}^{i,\ell}\}_{i,j,k,\ell}$ and $\mathcal{E}$ as in Equations \eqref{eq: event1} and \eqref{eq: event2}. Since each entry of the matrix $\mathbf{V}^{(i)} = \mathbf{V}_1^{(i)} \otimes \dots \otimes \mathbf{V}_D^{(i)}$ is a nonzero multiple of one of $\bv_1^{(i)}, \dots, \bv_d^{(i)},$ we are guaranteed that, on the complement of $\mathcal{E},$ for each pair of distinct indices $j,k \in [M]^D$, there exist two indices $i, \ell \in [n]$ such that $(\mathbf{V}^{(i)})_{j,j} - (\mathbf{V}^{(i)})_{k,k} \neq (\mathbf{V}^{(\ell)})_{j,j} - (\mathbf{V}^{(\ell)})_{k,k}$, and neither $(\mathbf{V}^{(i)})_{j,j} - (\mathbf{V}^{(i)})_{k,k}$ nor $(\mathbf{V}^{(\ell)})_{j,j} - (\mathbf{V}^{(\ell)})_{k,k}$ are equal to zero. Here, $\mathbf{V}^{(i)}_{j,j}$ denotes the $j^{\textrm{th}}$ entry of the diagonal matrix $\mathbf{V}^{(i)}.$ Following the argument used in Step 1 of the proof of Theorem \ref{thm: 2}, we see that on the complement of $\mathcal{E},$ the centralizer of $\{\mathbf{A}^{(1)}, \dots, \mathbf{A}^{(N)}\}$ is contained in the set of diagonal matrices, and that $\mathcal{E}$ occurs with probability at most $M(M-1) \left(1-2p^2(1-p^2) \right)^{N/2}.$ Following the argument in Step 2 of the proof of Theorem \ref{thm: 2}, to show that the centralizer of $\{\mathbf{A}^{(1)}, \dots, \mathbf{A}^{(N)}\}$ is trivial, it suffices to exhibit a permutation $\pi$ of the indices $[d]$ (with $d = M^d)$ such that $(\mathbf{\Delta}_{\textrm{FD},D})_{\pi(k),\pi(k+1)} \neq 0$ for all $k \in [d-1].$ To verify this, it suffices to show that $\mathbf{\Delta}_{\textrm{FD},D}$ has nonzero elements along its sub-diagonal (hence the permutation can be chosen as $\pi(k) = k+1$ for $k \in [d-1]$). To see this, recall the formula for the finite difference Laplace matrix in $D$-dimensions:
    \begin{align*}
         \mathbf{\Delta}_{\textrm{FD},D} = \sum_{i=1}^{D} \left(\bigotimes_{j=1}^{i-1} \mathbf{I}_{M} \right) \otimes \mathbf{\Delta}_{\textrm{FD},1} \otimes \left(\bigotimes_{j=i+1}^{N} \mathbf{I}_{M} \right).
    \end{align*}
    The summand given by
    \begin{align*}
        \mathbf{I}_M \otimes \dots \otimes \mathbf{I}_M \otimes \mathbf{\Delta}_{\textrm{FD},1} = \begin{bmatrix}
            \mathbf{\Delta}_{\textrm{FD},1} & & \\
            & \ddots & \\
            & & \mathbf{\Delta}_{\textrm{FD},1}
        \end{bmatrix} \in \R^{d \times d}
    \end{align*}
    has nonzero elements along its sub-diagonal, since $\mathbf{\Delta}_{\textrm{FD},1}$ has nonzero elements along its sub-diagonal. All other summands in the equation defining $\mathbf{\Delta}_{\textrm{FD},D}$ are of the form
    \begin{align*}
        \mathbf{I}_{k_1} \otimes \mathbf{\Delta}_{\textrm{FD},1} \otimes \mathbf{I}_{k_2} = \begin{bmatrix}
            \left(\mathbf{I}_{k_1} \otimes \mathbf{\Delta}_{\textrm{FD}} \right)_{1,1} \mathbf{I}_{k_2} & \dots & \left(\mathbf{I}_{k_1} \otimes \mathbf{\Delta}_{\textrm{FD}} \right)_{1,k_1M} \mathbf{I}_{k_2} \\
            \vdots & & \vdots \\
            \left(\mathbf{I}_{k_1} \otimes \mathbf{\Delta}_{\textrm{FD}} \right)_{k_1M,1} \mathbf{I}_{k_2} & \dots & \left(\mathbf{I}_{k_1} \otimes \mathbf{\Delta}_{\textrm{FD}} \right)_{k_1M,k_1M} \mathbf{I}_{k_2}
        \end{bmatrix}
    \end{align*}
    where $k_1, k_2 \in \mathbb{N}$ satisfy $k_1 k_2 = M^{D-1}.$ Clearly, the sub-diagonal entries of the matrix defined above are all zero. This proves that $\mathbf{\Delta}_{\textrm{FD},D}$ has nonzero entries along its sub-diagonal, and hence completes the proof.
\end{proof}

\begin{proof}[Proof of Theorem \ref{thm: FEM}]
    We recall that a sample from the finite element distribution $\p_{\textrm{FEM}}$ takes the form $\mathbf{A} = -\mathbf{\Delta}_{\textrm{FEM},1} + \mathbf{V}$, where $\mathbf{\Delta}_{\textrm{FEM}}$ is given by Equation \eqref{eq: FEMlaplace} and $\mathbf{V}$ is given by Equation \eqref{eq: FEMpotentialmatrix}, with $\bv_1, \dots, \bv_{M-1}$ being independent samples from $\textrm{Ber}_p\{a,b\}.$ We let $\mathbf{A}^{(1)}, \dots, \mathbf{A}^{(N)}$ be iid sampels from $\p_{\textrm{FEM}}$ and write $\mathbf{A}^{(i)} = -\mathbf{\Delta}_{\textrm{FEM},1} + \mathbf{V}^{(i)}.$ We aim to prove a lower bound on the probability that the centralizer of the set $\{\mathbf{A}^{(1)}, \dots, \mathbf{A}^{(N)}, \mathbf{\Delta}_{\textrm{FEM},1}\}$ is trivial. To this end, we verify the conditions of Assumption \ref{assum: main}. Since $\mathbf{\Delta}_{\textrm{FEM},1}$ is equal to $\mathbf{\Delta}_{\textrm{FD},1}$ up to a constant multiple depending on $M$, and we showed that $-\mathbf{\Delta}_{\textrm{FD},1}$ satisfies item 1 of Assumption \ref{assum: main} in the proof of Theorem \ref{thm: FD}, we know that $\mathbf{\Delta}_{\textrm{FEM},1}$ satisfies item 1 of Assumption \ref{assum: main}. To show that $\mathbf{V}$ satisfies item 2 of Assumption \ref{assum: main}, we recall from the proof of Theorem \ref{thm: FD} that the normalized vectors $\mathbf{u}_1, \dots, \mathbf{u}_M$ defined by $\mathbf{u}_k \propto \left( 1, \textrm{cos} \left(\frac{2\pi k}{M} \right), \dots, \cos \left(\frac{2\pi(d-1)k}{M} \right) \right),$ form a basis of eigenvectors of $\mathbf{\Delta}_{\textrm{FEM},1}.$ We will show that for each $k \in [M-1]$
    \begin{align}\label{eq: femprob}
        \p_{\textrm{FEM}} \left(\mathbf{u}_k^T \mathbf{V} \mathbf{u}_{k+1} = 0 \right) \leq \frac{1}{\sqrt{1+2p(1-p)}}.
    \end{align}
    Using the definition of $\mathbf{V}$ in Equation \eqref{eq: FEMpotentialmatrix}, we compute that
    \begin{align*}
        \mathbf{u}_k^T \mathbf{V} \mathbf{u}_{k+1} = \langle \bv, \mathbf{w}_k \rangle,
    \end{align*}
    where $\bv = (\bv_1, \dots, \bv_{M-1}) \in \R^{M-1}$ and $\mathbf{w}_k \in \R^{M-1}$ is given by
    \begin{align*}
        (\mathbf{w}_k)_j = \begin{cases}
            2 (\mathbf{u}_k)_1 (\mathbf{u}_{k+1})_1 + 4 (\mathbf{u}_k)_2 (\mathbf{u}_{k+1})_2 + (\mathbf{u}_k)_2 (\mathbf{u}_{k+1})_1 + (\mathbf{u}_k)_1 (\mathbf{u}_{k+1})_2, \; j = 1 \\
            2(\mathbf{u}_k)_M (\mathbf{u}_{k+1})_M + 4(\mathbf{u}_k)_{M-1} (\mathbf{u}_{k+1})_{M-1} + (\mathbf{u}_k)_M (\mathbf{u}_{k+1})_{M-1} + (\mathbf{u}_k)_{M-1} (\mathbf{u}_{k+1})_M, \; j = M-1 \\
            4 \left((\mathbf{u}_k)_j (\mathbf{u}_{k+1})_j + (\mathbf{u}_k)_{j+1} (\mathbf{u}_{k+1})_{j+1} \right) + (\mathbf{u}_k)_j (\mathbf{u}_{k+1})_{j+1} + (\mathbf{u}_k)_{j+1} (\mathbf{u}_{k+1})_j, \; \textrm{otherwise.}
        \end{cases}
    \end{align*}
    By Inequality \eqref{eq: bernoulliprob} in the proof of Theorem \ref{thm: FD}, we have
    \begin{align*}
        \p_{\textrm{FEM}} \left(\langle \bv, \mathbf{w}_k \rangle = 0 \right) \leq \frac{1}{\sqrt{1+2\|\mathbf{w}_0\| p(1-p)}},
    \end{align*}
    thus to prove inequality \eqref{eq: femprob}, it suffices to show that $\|\mathbf{w}_k\|_0 \geq 1$ for each $k \in [M-1].$ To prove this, we will show that $(\mathbf{w}_k)_1 \neq 0$ for all $k.$ Using the definition of the vectors $\mathbf{u}_1, \dots, \mathbf{u}_k$, we compute that
    \begin{align*}
        (\mathbf{w}_k)_1 \propto 2 + 4 \cos \left( \frac{2\pi k}{M} \right) \cos \left( \frac{2\pi(k+1)}{M} \right) + \cos \left( \frac{2 \pi k}{M} \right) + \cos \left(\frac{2\pi (k+1)}{M} \right),
    \end{align*}
    where '$\propto$' hides a constant factor that depends on $M$. An exercise in single-variable calculus shows that the function $x \mapsto 2 + 4\cos(2\pi x/M) \cos(2 \pi(x+1)/M) + \cos(2\pi x/M) + \cos(2\pi(x+1)/M)$ has no zeros when $M \geq 5$. This proves that $\|\mathbf{w}_k\|_0 \geq 1$ for each $k,$ and hence Inequality \eqref{eq: femprob} is established. The fact that
    \begin{align*}
        \p_{\textrm{FEM}} \left(\textrm{$\{\mathbf{A}^{(1)}, \dots, \mathbf{A}^{(N)}, \mathbf{\Delta}_{\textrm{FEM,1}}\}$ has trivial centralizer} \right) \leq 1-(M-1) \left( \frac{1}{\sqrt{1+2 p(1-p)}} \right)^N
    \end{align*}
    is then a direct consequence of Theorem \ref{thm: main}.
\end{proof}
Note that a stronger lower bound on $\|\mathbf{w}_k\|_0$ can improve the estimate; in particular, we conjecture that $\min_k \|\mathbf{w}_k\|_0$ scales linearly in $M$, which would substantially improve the result, but we leave it to future work to verify this.

\section{Acknowledgment}
YL and FC thank the support from the NSF CAREER Award DMS-2442463.

\bibliographystyle{abbrv}
\bibliography{refs}

@article{gray2006toeplitz,
  title={Toeplitz and circulant matrices: A review},
  author={Gray, Robert M and others},
  journal={Foundations and Trends{\textregistered} in Communications and Information Theory},
  volume={2},
  number={3},
  pages={155--239},
  year={2006},
  publisher={Now Publishers, Inc.}
}

@article{madiman2023bernoulli,
  title={Bernoulli sums and R{\'e}nyi entropy inequalities},
  author={Madiman, Mokshay and Melbourne, James and Roberto, Cyril},
  journal={Bernoulli},
  volume={29},
  number={2},
  pages={1578--1599},
  year={2023},
  publisher={Bernoulli Society for Mathematical Statistics and Probability}
}

@article{calvello2024continuum,
  title={Continuum attention for neural operators},
  author={Calvello, Edoardo and Kovachki, Nikola B and Levine, Matthew E and Stuart, Andrew M},
  journal={arXiv preprint arXiv:2406.06486},
  year={2024}
}

@article{yang2023context,
  title={In-context operator learning with data prompts for differential equation problems},
  author={Yang, Liu and Liu, Siting and Meng, Tingwei and Osher, Stanley J},
  journal={Proceedings of the National Academy of Sciences},
  volume={120},
  number={39},
  pages={e2310142120},
  year={2023},
  publisher={National Academy of Sciences}
}

@article{yang2023prompting,
  title={Prompting in-context operator learning with sensor data, equations, and natural language},
  author={Yang, Liu and Meng, Tingwei and Liu, Siting and Osher, Stanley J},
  journal={arXiv preprint arXiv:2308.05061},
  year={2023}
}

@inproceedings{liu2023does,
  title={Does in-context operator learning generalize to domain-shifted settings?},
  author={Liu, Jerry Weihong and Erichson, N Benjamin and Bhatia, Kush and Mahoney, Michael W and Re, Christopher},
  booktitle={The symbiosis of deep learning and differential equations III},
  year={2023}
}

@article{serrano2024zebra,
  title={Zebra: In-context and generative pretraining for solving parametric pdes},
  author={Serrano, Louis and Koupa{\"\i}, Armand Kassa{\"\i} and Wang, Thomas X and Erbacher, Pierre and Gallinari, Patrick},
  journal={arXiv preprint arXiv:2410.03437},
  year={2024}
}

@article{vaswani2017attention,
  title={Attention is all you need},
  author={Vaswani, Ashish and Shazeer, Noam and Parmar, Niki and Uszkoreit, Jakob and Jones, Llion and Gomez, Aidan N and Kaiser, {\L}ukasz and Polosukhin, Illia},
  journal={Advances in neural information processing systems},
  volume={30},
  year={2017}
}

@article{zhang2024trained,
  title={Trained transformers learn linear models in-context},
  author={Zhang, Ruiqi and Frei, Spencer and Bartlett, Peter L},
  journal={Journal of Machine Learning Research},
  volume={25},
  number={49},
  pages={1--55},
  year={2024}
}

@article{ahn2023transformers,
  title={Transformers learn to implement preconditioned gradient descent for in-context learning},
  author={Ahn, Kwangjun and Cheng, Xiang and Daneshmand, Hadi and Sra, Suvrit},
  journal={Advances in Neural Information Processing Systems},
  volume={36},
  pages={45614--45650},
  year={2023}
}

@article{edelman2005random,
  title={Random matrix theory},
  author={Edelman, Alan and Rao, N Raj},
  journal={Acta numerica},
  volume={14},
  pages={233--297},
  year={2005},
  publisher={Cambridge University Press}
}

@book{tao2012topics,
  title={Topics in random matrix theory},
  author={Tao, Terence},
  volume={132},
  year={2012},
  publisher={American Mathematical Soc.}
}

@article{cole2024context,
  title={In-Context Learning of Linear Systems: Generalization Theory and Applications to Operator Learning},
  author={Cole, Frank and Lu, Yulong and Xu, Wuzhe and Zhang, Tianhao},
  journal={arXiv preprint arXiv:2409.12293},
  year={2024}
}

@article{mishra2025continuum,
  title={Continuum Transformers Perform In-Context Learning by Operator Gradient Descent},
  author={Mishra, Abhiti and Patel, Yash and Tewari, Ambuj},
  journal={arXiv preprint arXiv:2505.17838},
  year={2025}
}

@article{garg2022can,
  title={What can transformers learn in-context? a case study of simple function classes},
  author={Garg, Shivam and Tsipras, Dimitris and Liang, Percy S and Valiant, Gregory},
  journal={Advances in neural information processing systems},
  volume={35},
  pages={30583--30598},
  year={2022}
}

@book{gantmakher2000theory,
  title={The theory of matrices},
  author={Gantmakher, Feliks Ruvimovich},
  volume={131},
  year={2000},
  publisher={American Mathematical Soc.}
}

@phdthesis{jeong2022linear,
  title={Linear Algebra, Random Matrices and Lie Theory},
  author={Jeong, Sungwoo},
  year={2022},
  school={Massachusetts Institute of Technology}
}

@article{palheta2022commutators,
  title={Commutators of random matrices from the unitary and orthogonal groups},
  author={Palheta, Pedro HS and Barbosa, Marcelo R and Novaes, Marcel},
  journal={Journal of Mathematical Physics},
  volume={63},
  number={11},
  year={2022},
  publisher={AIP Publishing}
}

@article{lu2025asymptotic,
  title={Asymptotic theory of in-context learning by linear attention},
  author={Lu, Yue M and Letey, Mary and Zavatone-Veth, Jacob A and Maiti, Anindita and Pehlevan, Cengiz},
  journal={Proceedings of the National Academy of Sciences},
  volume={122},
  number={28},
  pages={e2502599122},
  year={2025},
  publisher={National Academy of Sciences}
}

@article{mahankali2023one,
  title={One step of gradient descent is provably the optimal in-context learner with one layer of linear self-attention},
  author={Mahankali, Arvind and Hashimoto, Tatsunori B and Ma, Tengyu},
  journal={arXiv preprint arXiv:2307.03576},
  year={2023}
}

\end{document}